%% file: main.tex
\documentclass[sigconf]{acmart}

\copyrightyear{2026}
\acmYear{2026}
\setcopyright{cc}
\setcctype{by}
\acmConference[KDD '26]{Proceedings of the 32nd ACM SIGKDD Conference on Knowledge Discovery and Data Mining V.1}{August 09--13, 2026}{Jeju Island, Republic of Korea}
\acmBooktitle{Proceedings of the 32nd ACM SIGKDD Conference on Knowledge Discovery and Data Mining V.1 (KDD '26), August 09--13, 2026, Jeju Island, Republic of Korea}
\acmPrice{}
\acmDOI{10.1145/3770854.3780240}
\acmISBN{979-8-4007-2258-5/2026/08}

\AtBeginDocument{%
  }

\usepackage{mathtools}
\newtheorem{proposition}{Proposition}[section]
\usepackage{algorithm}
\usepackage{algorithmic}
\usepackage{graphicx} 
\usepackage{multirow}
\usepackage{subcaption}
\usepackage{enumitem}
\usepackage{makecell}

\usepackage{cancel}
\usepackage{balance}

\begin{document}

\title{Balanced Anomaly-guided Ego-graph Diffusion Model for Inductive Graph Anomaly Detection}


\author{Chunyu Wei}
\affiliation{%
  \institution{Renmin University of China}
  \city{Beijing}
  \country{China}
}
\authornotemark[1]
\email{weichunyu@ruc.edu.cn}

\author{Siyuan He}
\affiliation{%
  \institution{Renmin University of China}
  \city{Beijing}
  \country{China}
}
\authornote{Both authors contributed equally to this research.}
\email{river0307@ruc.edu.cn}

\author{Yu Wang}
\affiliation{%
  \institution{Independent Researcher}
  \city{Beijing}
  \country{China}
}
\email{feather1014@gmail.com}

\author{Yueguo Chen}
\affiliation{%
  \institution{Renmin University of China}
  \city{Beijing}
  \country{China}}
\email{chenyueguo@ruc.edu.cn}
\authornote{Corresponding author. He works at BRAIN of RUC.}

\author{Yunhai Wang}
\affiliation{%
 \institution{Renmin University of China}
 \city{Beijing}
 \country{China}}
\email{cloudseawang@gmail.com}

\author{Bing Bai}
\affiliation{%
  \institution{Microsoft}
  \city{Beijing}
  \country{China}}
\email{baibing12321@163.com}

\author{Yidong Zhang}
\affiliation{%
  \institution{Renmin University of China}
  \city{Beijing}
  \country{China}
}
\email{ydzhang@ruc.edu.cn}

\author{Yong Xie}
\affiliation{%
 \institution{Nanjing University of Posts and Telecommunications}
 \city{NanJing}
 \country{China}}
\email{yongxie@njupt.edu.cn}

\author{Shunming Zhang}
\affiliation{%
  \institution{Renmin University of China}
  \city{Beijing}
  \country{China}
}
\email{szhang@ruc.edu.cn}

\author{Fei Wang}
\affiliation{%
  \institution{Cornell University}
  \city{Storrs}
  \country{United States}}
\email{few2001@med.cornell.edu}

\renewcommand{\shortauthors}{Chunyu Wei et al.}

\begin{abstract}
Graph anomaly detection (GAD) is crucial in applications like fraud detection and cybersecurity. Despite recent advancements using graph neural networks (GNNs), two major challenges persist. At the model level, most methods adopt a transductive learning paradigm, which assumes static graph structures, making them unsuitable for dynamic, evolving networks. At the data level, the extreme class imbalance, where anomalous nodes are rare, leads to biased models that fail to generalize to unseen anomalies. These challenges are interdependent: static transductive frameworks limit effective data augmentation, while imbalance exacerbates model distortion in inductive learning settings. To address these challenges, we propose a novel data-centric framework that integrates dynamic graph modeling with balanced anomaly synthesis. 
Our framework features: (1) a discrete ego-graph diffusion model, which captures the local topology of anomalies to generate ego-graphs aligned with anomalous structural distribution, and (2) a curriculum anomaly augmentation mechanism, which dynamically adjusts synthetic data generation during training, focusing on underrepresented anomaly patterns to improve detection and generalization. Experiments on five datasets demonstrate that the effectiveness of our framework. 
\end{abstract}





\maketitle
\newcommand\kddavailabilityurl{https://doi.org/10.1145/3770854.3780240}
\ifdefempty{\kddavailabilityurl}{}{
\begingroup\small\noindent\raggedright\textbf{Resource Availability:}\\
The source code of this paper has been made publicly available at \url{https://github.com/OaxKnud/BAED}.

\endgroup
}

\input{data/introduction}

\input{data/related_work}

\input{data/preliminaries}
\input{data/methodology}
\input{data/analysis}

\input{data/experiment}
\input{data/conclusion}

\begin{acks}
This research was supported by the National Key R\&D Program of China (No. 2023YFC3304701), NSFC (No.6250072448, No.62272466, U24A20233). It was
also supported by the Big Data and Responsible Artificial Intelligence for National Governance, Renmin University of China, and the Joint Academy on Future Humanity, Renmin University of China and Westlake University.
\end{acks}

\nocite{*}
\bibliographystyle{ACM-Reference-Format}
\balance
\bibliography{sample-base}

\appendix
\input{data/appendix}

\end{document}

%% file: data/introduction.tex
\section{Introduction}
\label{sec:introduction}

\begin{figure}[!t]
  \centering
  \includegraphics[width=0.99\linewidth]{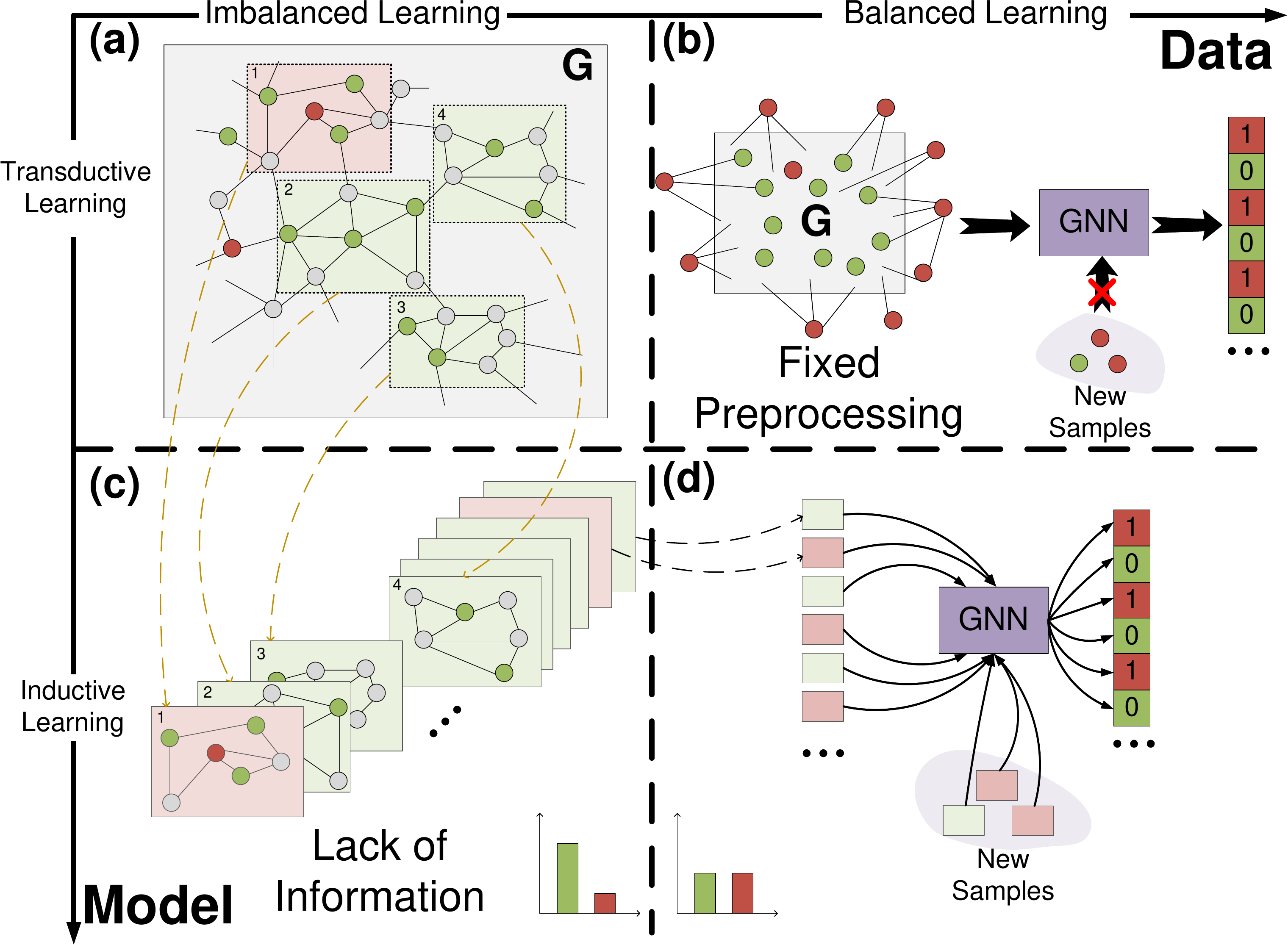}
  \caption{(a) \underline{Vanilla GAD paradigm}: transductive learning is performed on the whole graph with imbalanced labels.
(b) \underline{Graph-level data augmentation}: fixed preprocessing that require retraining the model whenever new samples are introduced, limiting adaptability during the training process.
(c) \underline{Inductive GAD on ego-graphs}: inductive learning on ego-graphs inherits the label imbalance issue from the original graph.
(d) \underline{Dynamic and balanced ego-graph augmentation}: our proposed inductive framework dynamically adjusts the type and ratio of generated samples during training to address imbalance and enhance model adaptability in real-time.}
  \label{fig:motivation}
\end{figure}

Graph anomaly detection (GAD) is essential in various applications, as networks are often modeled as graphs with nodes representing entities and edges representing relationships~\cite{DBLP:journals/tsc/WeiFJZ24,DBLP:conf/icann/ZhangWYFJ24}.
It has gained significant attention in fields like fraud detection~\cite{DBLP:conf/icdm/WangQL0JWFYZY19, DBLP:journals/eswa/MotieR24, DBLP:journals/tkde/ChengWZZ22}, social network analysis~\cite{DBLP:journals/sigkdd/YuQWLL16,DBLP:journals/tnsm/WeiFZ22,DBLP:journals/aei/ZhangYFZYW25}, and cybersecurity~\cite{DBLP:journals/tifs/WangZ22}. 
Unlike traditional anomaly detection tasks, where anomalies are typically identified by observing the features of individual nodes, GAD requires consideration not only of node attributes but also the graph structure~\cite{DBLP:conf/kdd/WeiHHWCB025,DBLP:conf/kdd/WeiLLDLW23}. 
This is because the relationships between nodes (i.e., the topology) often hold critical information that can reveal anomalous behaviors. For example, a node may disguise itself by modifying its features, but the interaction patterns represented by its topology can still highlight its anomalies.

Various Graph Neural Network (GNN)-based methods have been proposed to address graph anomaly detection by learning node representations that capture both feature and topological information. However, these methods face two main limitations: one at the \textbf{model level} and another at the \textbf{data level}, as illustrated in Figure~\ref{fig:motivation}(a).

At the \textbf{model level}, most approaches rely on transductive learning, where models are trained on a fixed set of labeled graph data, and the learned representations are subsequently used to detect anomalies. 
However, this approach assumes static graph structures, which is often unrealistic as real-world networks are dynamic~\cite{DBLP:conf/cikm/WeiLBL22}.
When the graph structure changes, the model must be retrained from scratch, a process that is both computationally expensive and impractical for large-scale or real-time applications~\cite{DBLP:journals/tsc/JiaFZWYW23,DBLP:journals/aei/JiaFWY23,DBLP:journals/tsc/WeiFZ23}. This limitation highlights a critical mismatch between the static assumptions of transductive learning and the inherently dynamic nature of real-world networks~\cite{DBLP:journals/tnsm/WeiFZJY24,DBLP:conf/www/WeiBBW22}.

At the \textbf{data level}, class imbalance is a significant challenge, as anomalous nodes are rare and their distribution is skewed. For instance, in the DGraph~\cite{huang2022dgraph}  dataset, only 1.3\% of nodes are anomalous. This imbalance leads to biased models that fail to capture the distinguishing characteristics of anomalous nodes, ultimately resulting in poor generalization during downstream prediction. 

These two challenges, transductive learning and imbalanced learning, are intrinsically linked. 
On one hand, the shift from transductive to inductive learning exacerbates the imbalance issue as illustrated in Figure~\ref{fig:motivation}(c).
This is because inductive learning requires generalization to unseen graph data, and the scarcity of anomalous nodes greatly hampers this ability. 
On the other hand, existing methods to address the imbalance through data synthesis often rely on pre-processing techniques, which do not dynamically adapt during model training, as illustrated in Figure~\ref{fig:motivation}(b). 
In transductive settings, the model's inability to adapt to graph structure changes exacerbates the complexity of the problem.

To simultaneously tackle these two challenges, we propose a novel data-centric framework,  \underline{B}alanced \underline{A}nomaly-guided \underline{E}go-graph \underline{D}iffusion (BAED) model, designed to address both the dynamic nature of graphs and the imbalance issue in GAD tasks, as illustrated in Figure~\ref{fig:motivation}(d). 
Specifically, we introduce a anomalous ego-graph synthesizer aimed at enhancing inductive graph anomaly detection by generating realistic, anomaly-specific ego-graphs that can be used to augment the training process. 
Our framework operates on the principle that effective anomaly detection requires not only a robust model but also the generation of informative samples that guide the model towards better generalization.

To achieve this, we focus on two key goals:

    \textbf{Better Fitting}: We propose a discrete ego-graph diffusion model that generates ego-graphs more aligned with the distribution of anomalous nodes. Unlike traditional diffusion models that map the entire graph to a latent space and overlook critical local neighborhood information, our model operates directly on the graph’s topology, preserving the local structure of anomalous nodes.
    
    \textbf{Better Adaptation}: 
    We introduce a curriculum anomaly augmentation mechanism that dynamically adjusts during training. Unlike traditional static pre-processing methods,
    our approach adapts to the training process itself. At each iteration, based on the model's previous loss, the synthesis process adjusts the distribution of generated anomaly samples to focus on underrepresented types of anomalies. 
    This enhances the model's ability to detect specific anomalies, improving both accuracy and generalization.

Our main contributions can be summarized as follows:
\begin{itemize}
    \item We propose a data-centric framework that simultaneously addresses the challenges of transductive learning and imbalanced learning in graph anomaly detection.
    \item We introduce a novel discrete ego-graph diffusion model that effectively generates realistic ego-graphs for curriculum anomaly augmentation, which enabling dynamic adaptation of the synthetic data generation during training.
    \item Through experiments on benchmarks for graph anomaly detection, BAED achieves state-of-the-art results. 
\end{itemize}

%% file: data/related_work.tex
\section{Related Works}
\label{sec:related_work}
\subsection{Graph Anomaly Detection}
With the advent of graph neural networks (GNNs)~\cite{DBLP:conf/iclr/XuHLJ19, DBLP:conf/kdd/WangC016, DBLP:conf/kdd/QiuTMDW018, DBLP:conf/nips/WeiLLW22, DBLP:conf/icml/WeiWBNBF23}, 
significant progress has been made in research on anomaly detection using node features and graph topology~\cite{DBLP:conf/sdm/DingLBL19,DBLP:journals/tnn/LiuLPGZK22,DBLP:conf/ijcai/PengLLLZ18}.
Traditional model-centric approaches predominantly rely on discriminative techniques, where representations are learned over the graph, followed by the construction of decision boundaries based on node embeddings to classify anomalies~\cite{DBLP:conf/kdd/WeiHHWCB025}. 
These methods often make use of partial label information and are typically framed in a semi-supervised or supervised learning setting. 
Notable examples include FdGars~\cite{DBLP:conf/www/WangWWHX19}, which employs multi-layer GNNs to classify users based on their content and behavior, and CARE-GNN~\cite{DBLP:conf/cikm/DouL0DPY20}, which adjusts neighbor aggregation thresholds via reinforcement learning to address inconsistencies. 
Similarly, approaches such as FRAUDRE~\cite{DBLP:conf/icdm/Zhang00BX0S21} and PC-GNN~\cite{DBLP:conf/www/LiuAQCFYH21} mitigate class imbalance by adopting specialized loss functions or label-balanced sampling strategies.

Spectral-based methods, such as BWGNN~\cite{DBLP:conf/icml/TangLGL22} and AMNet~\cite{DBLP:conf/ijcai/ChaiY0PXCJ22}, explore the graph’s frequency spectrum to capture anomalies, using techniques like Beta-graph wavelets and Bernstein polynomials to separate normal and abnormal frequency bands. 
In parallel, GHRN~\cite{DBLP:conf/www/GaoW0LF023} eliminates harmful heterogeneous edges to enhance model robustness, and SEC-GFD~\cite{DBLP:conf/aaai/XuWWWZW24} tackles heterophily and label imbalance through spectral filtering.

Unsupervised learning techniques, essential due to the scarcity of labeled data, have gained prominence in graph anomaly detection (GAD). These include traditional outlier detection methods like LOF~\cite{DBLP:conf/aaai/BandyopadhyayLM19} and Isolation Forest (IF)~\cite{DBLP:journals/tkdd/LiuTZ12}, as well as autoencoder-based models like MAPLE~\cite{DBLP:conf/pricai/SakuradaY14} and variational graph autoencoders (VGAEs)~\cite{DBLP:journals/corr/KipfW16a, DBLP:journals/tim/ZhangCXCNZHX23}, which improve anomaly detection by learning robust node representations. Structural similarity-based methods like SCAN~\cite{DBLP:conf/kdd/XuYFS07} and Radar~\cite{DBLP:conf/ijcai/LiDHL17} incorporate both attribute and structural information, while BOND~\cite{DBLP:conf/nips/LiuDZDHZDCPSSLC22} benchmarks various unsupervised anomaly detection methods.

To address the challenges of data scarcity and extreme imbalance in anomaly detection scenarios, various methods have adopted data augmentation strategies. For instance, some approaches enhance node features, generate contrastive views, or modify egonet structures~\cite{DBLP:conf/iclr/ParkSY22, DBLP:conf/wsdm/ZhaoZW21}, often tailored for specific algorithms~\cite{DBLP:conf/wsdm/ZhaoZW21, DBLP:journals/tkde/ZhengJLCPC23}. Recently, novel data-centric augmentation methods have emerged, directly tackling the issue of label scarcity by generating additional training data~\cite{DBLP:conf/kdd/0002LLDY024}. 
Our work falls into this category. 
However, while these methods generate additional node features, the graph's fixed topological structure, which is constrained by transductive approaches, remains unchanged throughout the augmentation process.
Our method is designed to address this issue by considering both dynamic graph structures and newly added nodes, ensuring better adaptability and generalization to GAD tasks.

\subsection{Inductive Leaning}
Inductive learning on graphs generalizes to unseen nodes or new graphs, essential for dynamic graph data.

\textbf{Network Embedding-Based Methods.}
Embedding methods extend static techniques like Skip-gram to dynamic settings, capturing temporal dynamics (e.g., Dynamic Skip-gram~\cite{DBLP:conf/ijcai/DuWSLW18}, temporal embeddings~\cite{DBLP:conf/www/NguyenLRAKK18, DBLP:conf/kdd/ZuoLLGHW18}). Autoencoder-based methods like DynGEM~\cite{DBLP:journals/corr/abs-1805-11273} and dyngraph2vec~\cite{DBLP:journals/kbs/GoyalCC20} excel in link prediction and anomaly detection but struggle to generalize to new graphs and lack scalability in highly dynamic environments.

\textbf{Graph Neural Network-Based Methods.}
GNNs support inductive learning by utilizing node features and structural information. Methods like GraphSAGE~\cite{DBLP:conf/nips/HamiltonYL17} introduce sampling-aggregation for new nodes and graphs, while TGAT~\cite{DBLP:conf/iclr/XuRKKA20} captures dynamic patterns with temporal encoding. Recent pre-training strategies (e.g., L2P-GNN~\cite{DBLP:conf/aaai/LuJ0S21}, graph-level pre-training~\cite{DBLP:conf/iclr/HuLGZLPL20}) improve inductive learning by learning transferable representations, though they still face challenges in aligning pre-training with fine-tuning objectives.

%% file: data/preliminaries.tex
\section{Preliminaries}
\label{sec:preliminaries}
\textbf{Attributed Graph.}
An \textit{attributed graph} is defined as a structure consisting of nodes, their associated features, and the relations (edges) between them. Formally, we represent an attributed graph as $\mathcal{G} = \{\mathcal{V}, \mathcal{E}, \mathbf{X}\}$, where:
\begin{itemize}[leftmargin=*]
    \item $\mathcal{V} = \{v_1, v_2, \ldots, v_N\}$ denotes the set of $N$ nodes.
    \item $\mathcal{E} = \{e_{ij} \mid v_i \text{ and } v_j \text{ are connected}, v_i, v_j \in \mathcal{V}\}$ represents the set of edges between nodes.
    \item $\mathbf{X} \in \mathbb{R}^{N \times d}$ represents the feature matrix of the graph, where each row $\mathbf{x}_i \in \mathbb{R}^d$ corresponds to the $d$-dimensional feature vector of node $v_i$.
\end{itemize}

To facilitate matrix-based computations, we define the graph's \textit{adjacency matrix} $\mathbf{A} \in \mathbb{N}^{N \times N}$, where each entry $\mathbf{A}_{ij}$ is binary:
{\small
\begin{equation}
\mathbf{A}_{ij} = 
\begin{cases} 
1, & \text{if there exists an edge } e_{ij} \in \mathcal{E}, \\
0, & \text{otherwise}.
\end{cases}
\end{equation}
}

\textbf{Ego-graph.}
An \textit{Ego-graph} is a specialized subgraph structure centered around a specific node $v_i$, referred to as the \textit{ego}, and its neighborhood $\mathcal{N}^{K}(i)$ within a predefined number of hops $K$. Formally, the node set of the ego-graph can be expressed as: $\mathcal{V}_{K}^{i} = \{v_i\} \cup \mathcal{N}^{K}(i)$,
where $\mathcal{N}^{K}(i)$ denotes the set of nodes reachable from $v_i$ within $K$ hops.
The ego-graph $\mathcal{G}_{K}^{i}$ is then defined as:
{\small
\begin{equation*}
\mathcal{G}_{K}^{i} = \big\{ \mathcal{V}_{K}^{i}, \{e_{jk} \mid e_{jk} \in \mathcal{E}, v_j, v_k \in \mathcal{V}_{K}^{i} \}, [x_k, \ldots]~\text{where}~v_k \in \mathcal{V}_{K}^{i} \big\}, 
\end{equation*}
}where $[x_k, \ldots]$ represents the set of feature vectors corresponding to each node $v_k \in \mathcal{V}_{K}^{i}$.

\textbf{Anomaly Detection on Graph.}
Given an \textit{attributed graph} $\mathcal{G}$, the node set $\mathcal{V}$ can be partitioned into two disjoint subsets, namely the anomalous node set $\mathcal{V}_a$ and the normal node set $\mathcal{V}_n$, such that: $\mathcal{V}_a \cap \mathcal{V}_n = \emptyset \quad \text{and} \quad \mathcal{V}_a \cup \mathcal{V}_n = \mathcal{V}$.

However, only a subset of nodes in the graph are labeled. Let $\mathcal{V}_{train} \subset \mathcal{V}$ denote the set of labeled nodes, and their corresponding labels are given as: $Y_{train} = \{y_i \mid y_i \in \{0, 1\}, v_i \in \mathcal{V}_{train} \}$, where $y_i = 1$ indicates that the node $v_i$ is anomalous ($v_i \in \mathcal{V}_a$), and $y_i = 0$ indicates that the node is normal ($v_i \in \mathcal{V}_n$).

The objective of \textit{graph anomaly detection} (GAD) is to infer the anomaly labels $Y_{test}$ of the unlabeled nodes $\mathcal{V}_{test} = \mathcal{V} \setminus \mathcal{V}_{train}$ by estimating the probability:
{\small
\begin{equation}
    p(v_i \in \mathcal{V}_a \mid \mathcal{G}, v_i \in \mathcal{V}_{test}),
\end{equation}
}which represents the likelihood that node $v_i$ is anomalous given the graph $\mathcal{G}$ and the partial node labels in $\mathcal{V}_{train}$.

\textbf{Inductive Anomaly Detection on Graph.}
Unlike traditional transductive approaches, which utilize the entire graph structure and node features, \textit{inductive anomaly detection} focuses on local information within the node’s ego-graph to compute the anomaly probability.
The goal of \textit{inductive graph anomaly detection} is to estimate the probability:
{\small
\begin{equation}
\label{equ:def_inductive}
    p(v_i \in \mathcal{V}_a \mid \mathcal{G}_{K}^{i}, v_i \in \mathcal{V}_{test}),
\end{equation}
}which represents the likelihood that the node $v_i$ is anomalous, conditioned only on its local ego-graph $\mathcal{G}_{K}^{i}$. 

%% file: data/methodology.tex
\begin{figure*}[!htbp]
  \centering
  \includegraphics[width=0.97\textwidth]{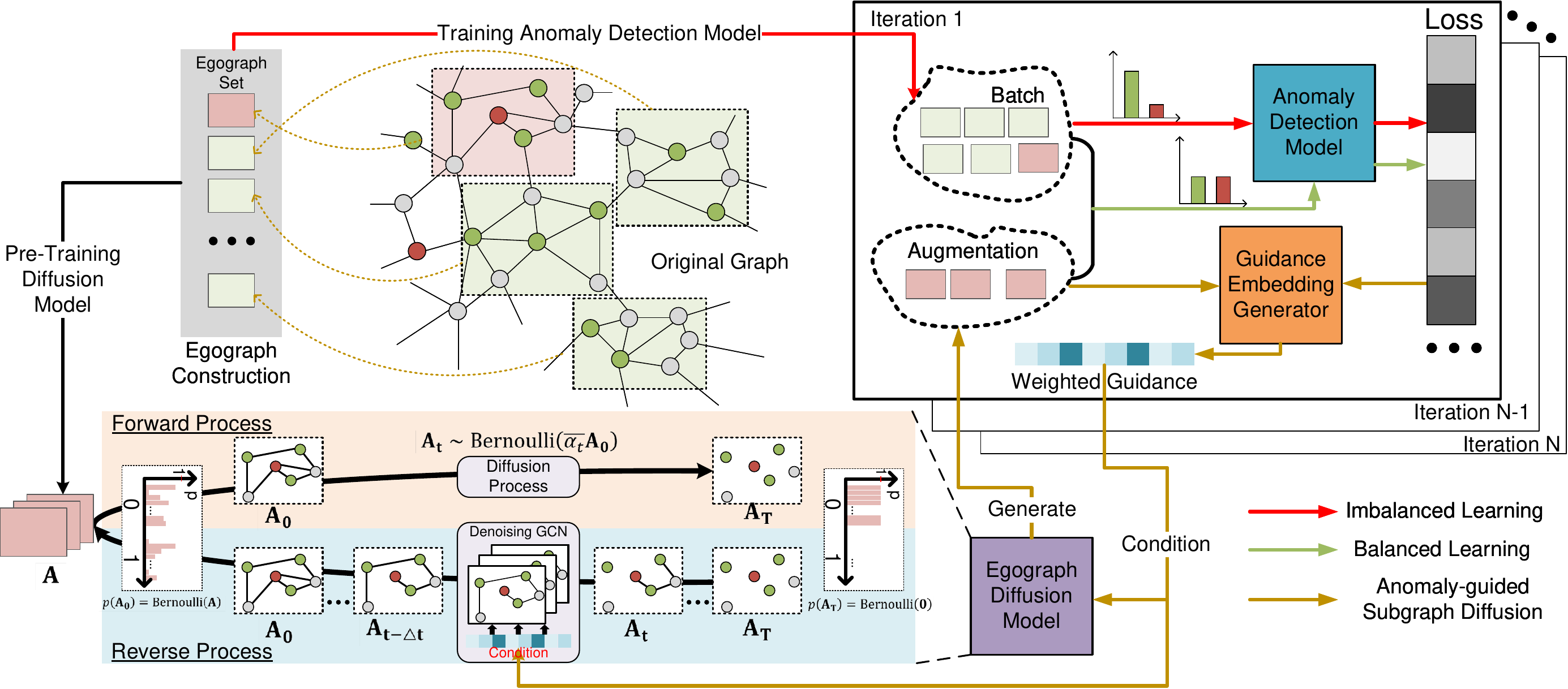}
    \caption{The BAED framework. \textbf{Left}: Pre-training the discrete ego-graph diffusion model with forward noise addition and reverse denoising processes. \textbf{Right}: Training iterations where the anomaly detection model processes both original imbalanced batch and augmented samples. The Guidance Embedding Generator (GIN) encodes anomalous ego-graphs into guidance embeddings, which are dynamically weighted based on previous losses to focus on underrepresented anomaly types. }
  \label{fig:framework}
\end{figure*}

\section{Methodology}
\label{sec:methodology}

We propose the Balanced Anomaly-guided Ego-graph Diffusion (BAED) framework, as shown in Figure~\ref{fig:framework}. The core idea is to use a discrete \textbf{Ego-graph Diffusion Model} to generate anomalous samples during the inductive GAD training, ensuring a balanced ratio of anomalous to normal samples. To improve adaptability, we introduce an \textbf{Anomaly-Guided Graph Generation} mechanism, controlled by a guidance embedding to align generated samples with specific anomaly patterns. Additionally, our \textbf{Curriculum Anomaly Augmentation} dynamically adjusts the anomaly patterns and proportions based on prior training results. Finally, we provide theoretical analyses demonstrating BAED's effectiveness.

\subsection{Inductive anomaly detection model}
In Equation~\ref{equ:def_inductive}, the goal is to estimate the anomaly probability $p(v_i \in \mathcal{V}_a \mid \mathcal{G}_{K}^{i}, v_i \in \mathcal{V}_{test})$. Our approach offers plug-and-play flexibility, allowing the adaptation of any graph operator-based transductive GNN. For training, we use the GAE (Graph Autoencoder) framework to train our selected GNN on the entire graph\cite{DBLP:conf/icml/WeiWBNBF23,DBLP:conf/icassp/FanZL20}. During inference, the model is applied only to the local ego-graph for node $v_i$: $p(v_i \in \mathcal{V}_a \mid \mathcal{G}_{K}^{i}) = GNN(\mathcal{G}_{K}^{i})$.

In the $l$-th layer of the GNN, node representations are updated based on node features and graph topology:
{\small
\begin{equation}
    \mathbf{E}^{(l)} = GCN(\mathbf{E}^{(l-1)}, \mathcal{G}_{K}^{i}),
\end{equation}
}
where $\mathbf{E}^{(l)}$ is the node embedding at layer $l$, and $GCN(\cdot)$ is the graph convolution operation. Initially, $\mathbf{E}^{(0)} = [x_k, \ldots]$, where $v_k \in \mathcal{V}_{K}^{i}$.
After $L$ layers of convolution, the final node representations for the ego-graph $\mathcal{G}_{K}^{i}$ are aggregated:
{\small
\begin{equation}
    \mathbf{H}_{K}^{i} = AGG(\mathbf{E}^{(1)}, \mathbf{E}^{(2)}, \ldots, \mathbf{E}^{(L)}),
\end{equation}
}where $AGG(\cdot)$ can be implemented using functions like mean or concatenation.

To improve anomaly detection, we propose a deviation-based representation for the final input:
{\small
\begin{equation}
    h_{\mathcal{G}_{K}^{i}} = \mathbf{H}_{K}^{i}[v_i] - mean\_pooling(\mathbf{H}_{K}^{i}),
\end{equation}
}where $\mathbf{H}_{K}^{i}[v_i]$ is the node embedding of $v_i$, and $mean\_pooling(\cdot)$ is the mean of all node embeddings. This deviation highlights differences between the ego node $v_i$ and its neighbors, reflecting its anomaly degree.

\begin{proposition}
The magnitude of the deviation $\|h_{\mathcal{G}_{K}^{i}}\|_2$ increases as the central node $v_i$ exhibits abnormal behavior compared to its local neighborhood. Formally, given the ego-graph $\mathcal{G}_{K}^{i}$, the deviation satisfies:
{\small
\begin{equation}
    \|h_{\mathcal{G}_{K}^{i}}\|_2 \propto \delta(v_i, \mathcal{N}^{K}(i)),
\end{equation}
}where $\delta(v_i, \mathcal{N}^{K}(i))$ denotes the degree of deviation of $v_i$ from the distribution of its neighbors $\mathcal{N}^{K}(i)$. 
\end{proposition}

Thus, the deviation operation provides a monotonically increasing measure of node $v_i$'s anomaly degree within its local ego-graph. We then use a Multi-Layer Perceptron (MLP)-based discriminator to map the representation $h_{\mathcal{G}_{K}^{i}}$ to a probability in the range $[0, 1]$, representing the likelihood that node $v_i$ is anomalous:
{\small
\begin{equation}
    p(v_i \in \mathcal{V}_a \mid \mathcal{G}_{K}^{i}) = MLP(h_{\mathcal{G}_{K}^{i}}).
\end{equation}
}

\subsection{Ego-graph Diffusion Model}
We focus on generating ego-graphs by manipulating the adjacency matrix through a discrete diffusion process. The process begins with the initial adjacency matrix $\mathbf{A}^0$, and noise is incrementally added at each step, represented as $q(\mathbf{A}^t \mid \mathbf{A}^{t-1})$, forming a trajectory $(\mathbf{A}^0, \mathbf{A}^1, \ldots, \mathbf{A}^T)$. The diffusion process models the joint distribution $q(\mathbf{A}^1, \ldots, \mathbf{A}^T \mid \mathbf{A}^0) = \prod_{t=1}^{T}q(\mathbf{A}^t \mid \mathbf{A}^{t-1})$. To model the reverse process, we introduce a denoising process $p_{\theta}(\mathbf{A}^{t-1} \mid \mathbf{A}^{t})$, parameterized by $\theta$, which is modeled as:
{\small
\begin{equation}
    p_{\theta}(\mathbf{A}^{0:t}) = p(\mathbf{A}^T) \prod_{t=1}^{T} p_{\theta}(\mathbf{A}^{t-1} \mid \mathbf{A}^t),
\end{equation}
}where $p(\mathbf{A}^T)$ is the distribution the adjacency matrix converges to after the forward diffusion.
An efficient diffusion model must satisfy three key properties:
\subsubsection{Closed-form Conditional Distribution}  
The conditional distribution $q(\mathbf{A}^t \mid \mathbf{A}^0)$ should have a closed-form solution to avoid recursive noise application during training. Following~\cite{DBLP:conf/iclr/VignacKSWCF23,DBLP:conf/icml/ChenHH023}, we use $\mathcal{B}(x ; \mu)$ to denote a Bernoulli distribution with parameter $\mu$. Each entry of the adjacency matrix, representing an edge, is treated as an independent random variable, with $\mathcal{B}(\mathbf{A} ; \mathbf{\mu})$ representing the entire adjacency matrix.

We define the forward process as resampling with probability $\beta_t$, following a Bernoulli distribution with parameter $p$. Thus, $q(\mathbf{A}^t \mid \mathbf{A}^{t-1})$ is an independent Bernoulli distribution for each edge:
{\small
\begin{equation}
\begin{aligned}
   \label{equ:forward}
   q\left(\mathbf{A}^t \mid \mathbf{A}^{t-1}\right) &= \prod_{i, j: i<j} \mathcal{B}\left(\mathbf{A}_{i, j}^t ; \left(1-\beta_t\right) \mathbf{A}_{i, j}^{t-1} + \beta_t p \right) \\
   &:= \mathcal{B}\left(\mathbf{A}^t ; \left(1-\beta_t\right) \mathbf{A}^{t-1} + \beta_t p \right).
\end{aligned}
\end{equation}
}This leads to: 
{\small
\begin{equation}
\label{equ:add_noise}
   q\left(\mathbf{A}^t \mid \mathbf{A}^0\right) = \mathcal{B}\left(\mathbf{A}^t ; \bar{\alpha}_t \mathbf{A}^0 + (1-\bar{\alpha}_t) p \right),
\end{equation}
}where $\alpha_t = 1 - \beta_t$ and $\bar{\alpha}_t = \prod_{\tau=1}^t \alpha_\tau$.

\subsubsection{Limiting Distribution}  
The second key property is that the limiting distribution $q_{\infty} = \lim_{T \to \infty} q(\mathbf{A}^T \mid \mathbf{A}^0)$ should be independent of $\mathbf{A}^0$, enabling its use as a prior. To achieve this, we use noise scheduling~\cite{DBLP:conf/nips/HoJA20} to adjust $\beta_1, \ldots, \beta_t$ such that $\bar{\alpha}_t$ approaches 0 as $t$ increases. This ensures $\lim_{T \to \infty} q\left(\mathbf{A}^T \mid \mathbf{A}^0\right) = \mathcal{B}\left(\mathbf{A}^T ; p\right) = q\left(\mathbf{A}^T\right)$, where $q(\mathbf{A}^T)$ is independent of $\mathbf{A}^0$. For simplicity, we set $p=0$.

\subsubsection{Closed-form Posterior Distribution}
   The third key property is that the posterior distribution $q(\mathbf{A}^{t-1} \mid \mathbf{A}^t, \mathbf{A}^0)$ should have a closed-form solution to serve as the target for the neural network in the reverse denoising process. The posterior of the forward transition conditioned on $\mathbf{A}^0$ can be computed as:
   {\small
   \begin{equation}
   q\left(\mathbf{A}^{t-1} \mid \mathbf{A}^t, \mathbf{A}^0\right) = \frac{q\left(\mathbf{A}^t \mid \mathbf{A}^{t-1}\right) q\left(\mathbf{A}^{t-1} \mid \mathbf{A}^0\right)}{q\left(\mathbf{A}^t \mid \mathbf{A}^0\right)}.
   \end{equation}
   }Since the posterior are trackable, it can be used to train the denoising model $p_\theta\left(\mathbf{A}^{t-1} \mid \mathbf{A}^t\right)$ by maximizing the evidence lower bound (ELBO) on the log-likelihood $\log p_\theta\left(\mathbf{A}^0\right)$.

\subsection{Anomaly-guided graph generation}
\label{sec:anomaly_guided}
Directly applying the denoising reverse process from a random noise distribution $\mathcal{B}\left(\mathbf{A}^T ; p\right)$ generates data that aligns with the original anomalous distribution. However, real-world anomalies often follow a multimodal distribution due to diverse anomaly patterns~\cite{DBLP:journals/jbd/RehmanB21}. When aiming to generate samples around a specific mode (e.g., a particular anomaly pattern) to enhance model training, a random reverse process is insufficient. To address this, we propose an \textbf{Anomaly-Guided Graph Generation} framework to generate augmented samples for the desired anomaly pattern.

For a given anomalous sample represented by an ego-graph $\mathcal{G}_{K}^{i}$, we hypothesize that a suitable graph embedding can capture its anomalous characteristics.

\begin{proposition}
\textbf{Distributional Separability of Local Ego-graph Embedding Space.} For node embeddings $\mathbf{z}_v$ in the embedding space, if the Ego-graph distributions $p(S_v)$ of different anomaly pattern satisfy separability conditions, then the embedding distribution $p(\mathbf{z}_v)$ also satisfies:
{\small
\begin{equation}
D_{KL}(p(\mathbf{z} \mid C_i) \parallel p(\mathbf{z} \mid C_j)) > \delta, \quad \forall i \neq j,
\end{equation}
}where $\delta > 0$ is the threshold for distributional separability. 
\end{proposition}

Graph Isomorphism theory~\cite{DBLP:conf/iclr/XuHLJ19} demonstrates that GNNs can approximate the ability of graph isomorphism testing. They aggregate information from neighboring nodes to capture the structural properties of graphs. This implies that embedding a node's local neighborhood can effectively capture subtle structural differences. Consequently, we employ the Graph Isomorphism Network (GIN) as an encoder to model the anomalous sample into a guidance embedding. Formally, we define: $\mathbf{h}^{g} = \text{GIN}(\mathcal{G}_{K}^{i})$, where $\mathbf{h}^{g} \in \mathbb{R}^d$.

\paragraph{Conditional Reverse Process}
Given the noise addition process defined in Equation~\ref{equ:forward}, we define the discrete conditional denoising reverse process as:
{\small
\begin{equation}
p_\theta\left(\mathbf{A}^{t-1} \mid \mathbf{A}^t, \mathbf{h}^{g} \right)
= \prod_{i, j: i<j} p_\theta\left(\mathbf{A}^{t-1}_{i,j} \mid \mathbf{A}^t_{i,j}, \mathbf{h}^{g} \right).
\end{equation}
}

\paragraph{Training the Denoising Network}
Following~\cite{DBLP:conf/icml/Sohl-DicksteinW15, DBLP:conf/nips/HoJA20}, we optimize the denoising model $p_\theta\left(\mathbf{A}^{t-1} \mid \mathbf{A}^t, \mathbf{h}^{g}\right)$ by maximizing the evidence lower bound (ELBO) of the log-likelihood $\log p_\theta\left(\mathbf{A}^0 \mid \mathbf{h}^{g} \right)$. The ELBO is formulated as:
{\small
\begin{equation}
\log p_\theta\left(\mathbf{A}^0 \mid \mathbf{h}^{g} \right) \geq \mathbb{E}_{q(\mathbf{A}^{1:T} \mid \mathbf{A}^0, \mathbf{h}^{g})} \left[ \log \frac{p_\theta (\mathbf{A}^{0:T} \mid \mathbf{h}^{g})}{q(\mathbf{A}^{1:T} \mid \mathbf{A}^0, \mathbf{h}^{g})} \right].
\end{equation}
}This leads to the following training loss:
{\small
\begin{equation}
\label{equ:loss}
\begin{aligned}
&\quad \mathcal{L}\left(\mathbf{A}^0 ; \theta\right) \\
&= \mathbb{E}_{q}\left[-\log p_\theta (\mathbf{A}^0 \mid \mathbf{A}^1, \mathbf{h}^{g})\right] + \mathbb{D}_{KL} \left( q(\mathbf{A}^T \mid \mathbf{A}^0, \mathbf{h}^{g}) \parallel p(\mathbf{A}^T \mid \mathbf{h}^{g}) \right)  \\
&+ \sum_{t=2}^{T} \mathbb{E}_{q}\left[\mathbb{D}_{KL} \left( q(\mathbf{A}^{t-1} \mid \mathbf{A}^t, \mathbf{A}^0, \mathbf{h}^{g}) \parallel p_\theta (\mathbf{A}^{t-1} \mid \mathbf{A}^t, \mathbf{h}^{g}) \right) \right].
\end{aligned}
\end{equation}
}

Here, the first term is the reconstruction loss, denoted as $\mathcal{L}_{recon}$; the second term is the prior matching loss, $\mathcal{L}_{prior}$; and the third term is the consistency loss, $\mathcal{L}_{con}$. 

\paragraph{Anomaly-Guided Inference}

After training the denoising network $p_\theta\left(\mathbf{A}^{t-1} \mid \mathbf{A}^t, \mathbf{h}^{g}\right)$, the inference process guided by a specific guidance embedding $\mathbf{h}^{g}$ is outlined in Algorithm~\ref{algo:inference}.

\begin{algorithm}
\caption{Inference for Anomaly Detection}
\label{algo:inference}
\begin{algorithmic}[1]
\STATE Initialize $\mathbf{A}_T \sim \mathcal{B}\left(\mathbf{A}^T ; p\right)$, the Anomaly-Guidance Embedding $\mathbf{h}^{g}$, and diffusion steps $T$.
\FOR {$t = T$ to $1$}
    \STATE Draw $\mathbf{A}_{t-1} \sim p_\theta\left(\mathbf{A}^{t-1} \mid \mathbf{A}^t, \mathbf{h}^{g} \right)$.
\ENDFOR
\RETURN $\mathbf{A}_0$.
\end{algorithmic}
\end{algorithm}

This anomaly-guided framework enables the generation of augmented samples that focus on specific anomalous patterns, improving the robustness and diversity of anomaly detection models.

\subsection{Curriculum Anomaly Augmentation}

In Section~\ref{sec:anomaly_guided}, we introduced how the Ego-graph Diffusion Model generates Ego-graphs representing specific anomaly patterns guided by a guidance embedding. However, during training, the learning ability for different anomaly patterns varies. Thus, we propose \textbf{Curriculum Anomaly Augmentation} strategy, which dynamically adjusts the generation of guidance embeddings to focus on specific anomaly patterns. This enables the anomaly-guided graph generation to adaptively generate augmented samples that match the model's curriculum learning needs at each training iteration.

Based on the principle above, we use the loss values to compute importance weights for the guidance embeddings. As discussed in Section~\ref{sec:anomaly_guided}, we use the GIN encoder to compute the guidance embedding for each anomalous ego-graph 
$\mathcal{G}_K^i$ during training iteration $t$. The importance weight for each embedding is computed using a heuristic sigmoid function:
{\small
\begin{equation}
\label{equ:guide_weight}
w_i = \frac{1}{1 + e^{-\alpha (l_i - \beta)}} \cdot h(t),
\end{equation}
}where $\alpha$ and $\beta$ are hyperparameters that control the sensitivity and shift of the weight in response to the loss values $l_i$ (the anomaly detection loss of the $i$-th sample in the batch). The term $h(t)$ is a temporal factor, designed as a monotonically increasing function of the training iteration: $h(t) = \frac{t}{T}$, where $T$ is the total number of training iterations. This design ensures that the generation model focuses less on hard-to-learn samples in the early training stages, gradually shifting its attention to these underfit samples in later stages.

Using the computed weights, we aggregate the guidance embeddings for all samples in the batch. The final guidance embedding $\tilde{\mathbf{h}}^{g}$ for the anomaly-guided generation is computed as:
{\small
\begin{equation}
\label{eq:embedding calculate}
\tilde{\mathbf{h}}^{g} = \sum_{i=1}^N \frac{{w}_i}{\sum_{j=1}^N {w}_j} \cdot \text{GIN}(\mathcal{G}_K^i),
\end{equation}
}where ${w}_i$ is the normalized weight for the $i$-th sample, $N$ is the number of anomaly in the batch, and $\text{GIN}(\mathcal{G}_K^i)$ represents the guidance embedding computed by the GIN encoder for the corresponding ego-graph $\mathcal{G}_K^i$.

This weighting and aggregation process enables the dynamic generation of augmented samples that are tailored to the model's learning progress at each training stage. By adaptively focusing on underfit anomaly patterns in later training stages, this strategy integrates a curriculum learning paradigm into the anomaly-guided graph generation framework. Furthermore, it ensures stable training by gradually increasing the difficulty of the generated samples as the model progresses through training.

%% file: data/analysis.tex
\subsection{Theoretical Analysis}
\label{sec:analysis}

We provide theoretical justification for how BAED addresses the critical challenge of class imbalance in graph anomaly detection. Our analysis demonstrates that balanced anomaly augmentation through ego-graph diffusion enables optimal decision boundaries for anomaly detection.

Consider the binary classification where $X \in \mathcal{X}$ represents the input features and $Y \in \{0, 1\}$ denotes the anomaly label. Let $\Theta = \{\theta : \mathcal{X} \to [0, 1]\}$ be the set of all classifiers, where $\theta(x)$ represents the predicted probability of anomaly. Following the framework in~\cite{imbalance}, we assume the data is generated i.i.d. from distribution $\mathbb{P}$, with $\pi = \mathbb{P}(Y = 1)$ denoting the prior probability of anomalies.

For any classifier $\theta$, we define the true positive rate as $\text{TP}(\theta) = \mathbb{P}(Y = 1, \theta = 1)$ and the classification rate as $\gamma(\theta) = \mathbb{P}(\theta = 1)$. The performance metric takes the general form:
{\small
\begin{equation}
\mathcal{L}(\theta) = \frac{c_0 + c_1 \cdot \text{TP}(\theta) + c_2 \cdot \gamma(\theta)}{d_0 + d_1 \cdot \text{TP}(\theta) + d_2 \cdot \gamma(\theta)},
\label{eq:metric}
\end{equation}
}where $c_i, d_i$ for $i \in \{0, 1, 2\}$ are constants determined by $\pi$ and the specific performance metric (e.g., F1-score, precision-recall).

\begin{theorem}[Optimal Decision Boundary under Imbalance]
\label{thm:boundary}
Under mild regularity conditions on $\mathbb{P}$, the optimal decision boundary $\delta^*$ for maximizing $\mathcal{L}(\theta)$ is given by:
{\small
\begin{equation}
\delta^* = \frac{d_2 \mathcal{L}^* - c_2}{c_1 - d_1 \mathcal{L}^*},
\end{equation}
}where $\mathcal{L}^* = \sup_{\theta \in \Theta} \mathcal{L}(\theta)$. The Bayes optimal classifier takes:
{\small
\begin{equation}
\theta^*(x) = \begin{cases}
\mathbb{1}[\eta(x) > \delta^*] & \text{if } c_1 > d_1 \mathcal{L}^* \\
\mathbb{1}[\eta(x) < \delta^*] & \text{if } c_1 < d_1 \mathcal{L}^*
\end{cases}
\end{equation}
}where $\eta(x) = \mathbb{P}(Y = 1 | X = x)$ is the posterior probability of anomaly.
\end{theorem}

This theorem from~\cite{imbalance} above reveals a critical insight: when $\pi$ is small (severe class imbalance), the optimal threshold $\delta^*$ deviates significantly from the conventional 0.5, leading to suboptimal performance if not properly addressed. 

\begin{proposition}[BAED's Balancing Effect]
\label{prop:balance}
Let $\tilde{\mathbb{P}}$ denote the augmented distribution after applying BAED's ego-graph generation, with $\tilde{\pi} = \tilde{\mathbb{P}}(Y = 1)$ being the augmented anomaly ratio. BAED's curriculum augmentation strategy ensures: 
{\small
\begin{equation}
|\tilde{\pi} - 0.5| < |\pi - 0.5|,
\end{equation}
}leading to an optimal threshold $\tilde{\delta}^*$ closer to 0.5, thereby improving the stability and generalization of the anomaly detector.
\end{proposition}

\begin{proof}
Let $N_a$ and $N_n$ denote the number of anomalous and normal samples in the original dataset, respectively, with $\pi = N_a/(N_a + N_n)$. At training iteration $t$, BAED generates $M_t$ synthetic anomalous ego-graphs based on the importance weights $w_i$ defined in Eq.~\eqref{eq:embedding calculate}. The effective anomaly ratio becomes:
{\small
\begin{equation}
\tilde{\pi}_t = \frac{N_a + M_t}{N_a + N_n + M_t}.
\end{equation}
}Since $w_i = \frac{1}{1 + e^{-\alpha(l_i - \beta)}} \cdot \frac{t}{T}$ increases with the loss $l_i$ and training progress $t/T$, underrepresented anomaly types receive higher generation priority. The curriculum mechanism ensures $M_t$ increases monotonically, driving $\tilde{\pi}_t$ towards 0.5. As $\pi < 0.5$ in imbalanced settings, we have $\tilde{\pi}_t > \pi$, thus $|\tilde{\pi}_t - 0.5| < |\pi - 0.5|$.
\end{proof}

By dynamically generating anomalous ego-graphs that target underrepresented patterns, BAED effectively transforms the learning problem from one with extreme imbalance ($\pi \ll 0.5$) to a more balanced scenario ($\tilde{\pi} \approx 0.5$), thereby achieving superior detection performance as validated in our experiments.

%% file: data/experiment.tex
\begin{table*}[t]
\footnotesize
\centering
\caption{Overall performance comparison on five datasets within inductive setting, with the bold numbers representing the best results and the underlined numbers indicating the most competitive results (RQ1)}
\label{tab:main_result}
\resizebox{\linewidth}{!}{
\begin{tabular}{l|ccccccccccccccc}
\toprule
\multirow{2}{*}{\textbf{Method}} & \multicolumn{3}{c}{\textbf{Elliptic}} & \multicolumn{3}{c}{\textbf{Reddit}} & \multicolumn{3}{c}{\textbf{Photo}} & \multicolumn{3}{c}{\textbf{T-Finance}} & \multicolumn{3}{c}{\textbf{Dgraph}} \\
\cmidrule(r){2-4} \cmidrule(r){5-7} \cmidrule(r){8-10} \cmidrule(r){11-13} \cmidrule(r){14-16}
 & \textbf{AUROC} & \textbf{AUPRC} & \textbf{F1} & \textbf{AUROC} & \textbf{AUPRC} & \textbf{F1} & \textbf{AUROC} & \textbf{AUPRC} & \textbf{F1} & \textbf{AUROC} & \textbf{AUPRC} & \textbf{F1} & \textbf{AUROC} & \textbf{AUPRC} & \textbf{F1} \\
\midrule
GCN          & 0.517$_{\pm 0.004}$ & 0.065$_{\pm 0.003}$ & 0.491$_{\pm 0.009}$ & 0.508$_{\pm 0.007}$ & 0.040$_{\pm 0.005}$ & 0.492$_{\pm 0.005}$ & 0.499$_{\pm 0.001}$ & 0.087$_{\pm 0.004}$ & 0.512$_{\pm 0.017}$ & 0.357$_{\pm 0.011}$ & 0.156$_{\pm 0.016}$ & 0.488$_{\pm 0.007}$ & 0.529$_{\pm 0.014}$ & \underline{0.013$_{\pm 0.005}$} & 0.497$_{\pm 0.009}$ \\
\quad + AEGIS    & 0.541$_{\pm 0.014}$ & 0.098$_{\pm 0.004}$ & 0.516$_{\pm 0.007}$ & 0.531$_{\pm 0.007}$ & 0.039$_{\pm 0.003}$ & 0.498$_{\pm 0.015}$ & \underline{0.585$_{\pm 0.008}$} & \underline{0.124$_{\pm 0.005}$} & \underline{0.523$_{\pm 0.012}$} & 0.579$_{\pm 0.005}$ & 0.048$_{\pm 0.003}$ & 0.537$_{\pm 0.018}$ & 0.550$_{\pm 0.009}$ & 0.005$_{\pm 0.001}$ & 0.495$_{\pm 0.012}$ \\
\quad + GGAD     & \underline{0.603$_{\pm 0.005}$} & \underline{0.112$_{\pm 0.004}$} & \underline{0.532$_{\pm 0.017}$} & 0.525$_{\pm 0.015}$ & 0.041$_{\pm 0.002}$ & 0.505$_{\pm 0.007}$ & 0.560$_{\pm 0.015}$ & 0.122$_{\pm 0.005}$ & 0.497$_{\pm 0.019}$ & 0.823$_{\pm 0.008}$ & 0.183$_{\pm 0.007}$ & 0.587$_{\pm 0.016}$ & \underline{0.558$_{\pm 0.007}$} & 0.006$_{\pm 0.001}$ & \textbf{0.531}$_{\pm \textbf{0.014}}$ \\
\quad + CGenGA & 0.595$_{\pm 0.011}$ & 0.075$_{\pm 0.004}$ & 0.523$_{\pm 0.006}$ & \underline{0.681$_{\pm 0.016}$} & \underline{0.047$_{\pm 0.003}$} & \underline{0.532$_{\pm 0.017}$} & 0.503$_{\pm 0.014}$ & 0.094$_{\pm 0.004}$ & 0.496$_{\pm 0.005}$ & \underline{0.847$_{\pm 0.019}$} & \underline{0.776$_{\pm 0.019}$} & \underline{0.618$_{\pm 0.019}$} & 0.549$_{\pm 0.015}$ & \underline{0.013$_{\pm 0.002}$} & 0.513$_{\pm 0.015}$ \\
\quad + BAED     & \textbf{0.671$_{\pm \textbf{0.007}}$} & \textbf{0.154}$_{\pm \textbf{0.006}}$ & \textbf{0.566}$_{\pm \textbf{0.007}}$ & \textbf{0.732}$_{\pm \textbf{0.004}}$ & \textbf{0.077}$_{\pm \textbf{0.005}}$ & \textbf{0.543}$_{\pm \textbf{0.011}}$ & \textbf{0.703}$_{\pm \textbf{0.009}}$ & \textbf{0.306}$_{\pm \textbf{0.007}}$ & \textbf{0.581}$_{\pm \textbf{0.015}}$ & \textbf{0.908}$_{\pm \textbf{0.008}}$ & \textbf{0.856}$_{\pm \textbf{0.003}}$ & \textbf{0.713}$_{\pm \textbf{0.009}}$ & \textbf{0.566}$_{\pm \textbf{0.006}}$ & \textbf{0.014}$_{\pm \textbf{0.002}}$ & \underline{0.521$_{\pm 0.007}$} \\
\midrule
BWGNN & \underline{0.637$_{\pm 0.006}$} & 0.064$_{\pm 0.003}$ & 0.505$_{\pm 0.007}$ & 0.520$_{\pm 0.017}$ & 0.036$_{\pm 0.007}$ & 0.517$_{\pm 0.007}$ & 0.613$_{\pm 0.004}$ & 0.165$_{\pm 0.008}$ & 0.558$_{\pm 0.019}$ & 0.208$_{\pm 0.008}$ & 0.133$_{\pm 0.006}$ & 0.301$_{\pm 0.015}$ & 0.580$_{\pm 0.006}$ & \underline{0.015$_{\pm 0.002}$} & \underline{0.519$_{\pm 0.009}$} \\
\quad + AEGIS    & 0.564$_{\pm 0.014}$ & \underline{0.100$_{\pm 0.005}$} & \underline{0.512$_{\pm 0.017}$} & 0.545$_{\pm 0.007}$ & 0.037$_{\pm 0.005}$ & 0.563$_{\pm 0.012}$ & 0.649$_{\pm 0.006}$ & 0.149$_{\pm 0.006}$ & 0.561$_{\pm 0.003}$ & 0.528$_{\pm 0.014}$ & 0.042$_{\pm 0.003}$ & 0.551$_{\pm 0.003}$ & 0.445$_{\pm 0.014}$ & 0.006$_{\pm 0.001}$ & 0.488$_{\pm 0.014}$ \\
\quad + GGAD     & 0.554$_{\pm 0.012}$ & 0.099$_{\pm 0.004}$ & 0.507$_{\pm 0.005}$ & 0.553$_{\pm 0.014}$ & 0.040$_{\pm 0.003}$ & 0.558$_{\pm 0.011}$ & 0.646$_{\pm 0.007}$ & 0.208$_{\pm 0.007}$ & 0.567$_{\pm 0.011}$ & \underline{0.770$_{\pm 0.018}$} & 0.098$_{\pm 0.005}$ & 0.572$_{\pm 0.018}$ & 0.594$_{\pm 0.016}$ & 0.008$_{\pm 0.001}$ & 0.511$_{\pm 0.015}$ \\
\quad + CGenGA & 0.585$_{\pm 0.019}$ & 0.073$_{\pm 0.004}$ & 0.502$_{\pm 0.002}$ & \underline{0.570$_{\pm 0.014}$} & \underline{0.045$_{\pm 0.001}$}& \underline{0.587$_{\pm 0.024}$} & \underline{0.751$_{\pm 0.008}$}& \underline{0.497$_{\pm 0.013}$} & \underline{0.613$_{\pm 0.004}$} & 0.627$_{\pm 0.016}$ & \underline{0.212$_{\pm 0.008}$} & \underline{0.581$_{\pm 0.019}$} & \textbf{0.635}$_{\pm \textbf{0.017}}$ & 0.012$_{\pm 0.002}$ & 0.601$_{\pm 0.018}$ \\
\quad + BAED     & \textbf{0.669}$_{\pm \textbf{0.016}}$ & \textbf{0.102}$_{\pm \textbf{0.005}}$ & \textbf{0.563}$_{\pm \textbf{0.009}}$ & \textbf{0.654}$_{\pm \textbf{0.016}}$ & \textbf{0.057}$_{\pm \textbf{0.004}}$ & \textbf{0.613}$_{\pm \textbf{0.013}}$ & \textbf{0.980}$_{\pm \textbf{0.003}}$ & \textbf{0.911}$_{\pm \textbf{0.011}}$ & \textbf{0.853}$_{\pm \textbf{0.012}}$ & \textbf{0.909}$_{\pm \textbf{0.019}}$ & \textbf{0.848}$_{\pm \textbf{0.009}}$ & \textbf{0.702}$_{\pm \textbf{0.021}}$ & \underline{0.599$_{\pm 0.016}$} & \textbf{0.016}$_{\pm \textbf{0.002}}$ & \textbf{0.522}$_{\pm \textbf{0.018}}$ \\
\midrule
GraphSAGE    & 0.505$_{\pm 0.012}$ & 0.054$_{\pm 0.003}$ & 0.507$_{\pm 0.015}$ & 0.546$_{\pm 0.014}$ & 0.037$_{\pm 0.003}$ & 0.508$_{\pm 0.002}$ & 0.501$_{\pm 0.013}$ & 0.085$_{\pm 0.004}$ & 0.503$_{\pm 0.026}$ & 0.554$_{\pm 0.007}$ & 0.238$_{\pm 0.008}$ & 0.532$_{\pm 0.017}$ & 0.464$_{\pm 0.014}$ & \underline{0.012$_{\pm 0.001}$} & 0.497$_{\pm 0.008}$ \\
\quad + AEGIS    & 0.576$_{\pm 0.024}$ & 0.103$_{\pm 0.005}$ & 0.521$_{\pm 0.017}$ & 0.553$_{\pm 0.014}$ & 0.048$_{\pm 0.003}$ & \underline{0.525$_{\pm 0.017}$} & 0.628$_{\pm 0.016}$ & 0.128$_{\pm 0.006}$ & 0.562$_{\pm 0.030}$ & 0.533$_{\pm 0.003}$ & 0.037$_{\pm 0.003}$ & 0.521$_{\pm 0.011}$ & \underline{0.540$_{\pm 0.012}$} & 0.005$_{\pm 0.001}$ & \underline{0.509$_{\pm 0.017}$} \\
\quad + GGAD     & \underline{0.579$_{\pm 0.015}$} & \underline{0.105$_{\pm 0.005}$} & \underline{0.534$_{\pm 0.011}$} & 0.533$_{\pm 0.013}$ & 0.043$_{\pm 0.003}$ & 0.509$_{\pm 0.013}$ & \underline{0.669$_{\pm 0.012}$} & \underline{0.175$_{\pm 0.007}$} & \underline{0.572$_{\pm 0.011}$} & \underline{0.689$_{\pm 0.017}$} & 0.074$_{\pm 0.004}$ & 0.562$_{\pm 0.005}$ & 0.537$_{\pm 0.015}$ & 0.005$_{\pm 0.001}$ & 0.502$_{\pm 0.016}$ \\
\quad + CGenGA & 0.531$_{\pm 0.033}$ & 0.061$_{\pm 0.003}$ & 0.529$_{\pm 0.015}$ & \underline{0.624$_{\pm 0.015}$} & 0.053$_{\pm 0.004}$ & 0.522$_{\pm 0.019}$ & 0.544$_{\pm 0.014}$ & 0.088$_{\pm 0.004}$ & 0.513$_{\pm 0.011}$ & 0.596$_{\pm 0.015}$ & \underline{0.324$_{\pm 0.009}$} & \underline{0.563$_{\pm 0.009}$} & 0.529$_{\pm 0.015}$ & 0.009$_{\pm 0.001}$ & 0.498$_{\pm 0.013}$ \\
\quad + BAED     & \textbf{0.593}$_{\pm \textbf{0.015}}$ & \textbf{0.134}$_{\pm \textbf{0.006}}$ & \textbf{0.583}$_{\pm \textbf{0.013}}$ & \textbf{0.667}$_{\pm \textbf{0.016}}$ & \textbf{0.082}$_{\pm \textbf{0.005}}$ & \textbf{0.562}$_{\pm \textbf{0.002}}$ & \textbf{0.971}$_{\pm \textbf{0.020}}$ & \textbf{0.866}$_{\pm \textbf{0.003}}$ & \textbf{0.763}$_{\pm \textbf{0.015}}$ & \textbf{0.750}$_{\pm \textbf{0.018}}$ & \textbf{0.425}$_{\pm \textbf{0.012}}$ & \textbf{0.612}$_{\pm \textbf{0.017}}$ & \textbf{0.668}$_{\pm \textbf{0.019}}$ & \textbf{0.020}$_{\pm \textbf{0.002}}$ & \textbf{0.527}$_{\pm \textbf{0.019}}$ \\
\midrule
BernNet      & 0.495$_{\pm 0.012}$ & 0.053$_{\pm 0.007}$ & 0.402$_{\pm 0.014}$ & 0.469$_{\pm 0.013}$ & \underline{0.068$_{\pm 0.004}$} & 0.482$_{\pm 0.018}$ & 0.587$_{\pm 0.015}$ & 0.119$_{\pm 0.005}$ & 0.509$_{\pm 0.028}$ & 0.224$_{\pm 0.009}$ & 0.135$_{\pm 0.006}$ & 0.292$_{\pm 0.014}$ & 0.437$_{\pm 0.014}$ & 0.003$_{\pm 0.001}$ & 0.431$_{\pm 0.015}$ \\
\quad + AEGIS    & \underline{0.666$_{\pm 0.016}$} & \underline{0.139$_{\pm 0.005}$} & \underline{0.522$_{\pm 0.011}$} & 0.525$_{\pm 0.013}$ & 0.037$_{\pm 0.007}$ & 0.511$_{\pm 0.022}$ & 0.618$_{\pm 0.017}$ & 0.123$_{\pm 0.005}$ & 0.552$_{\pm 0.013}$ & 0.543$_{\pm 0.015}$ & 0.044$_{\pm 0.003}$ & 0.492$_{\pm 0.017}$ & 0.469$_{\pm 0.014}$ & 0.005$_{\pm 0.001}$ & 0.472$_{\pm 0.016}$ \\
\quad + GGAD     & 0.653$_{\pm 0.029}$ & 0.134$_{\pm 0.006}$ & 0.513$_{\pm 0.014}$ & 0.515$_{\pm 0.013}$ & 0.036$_{\pm 0.003}$ & 0.529$_{\pm 0.011}$ & 0.617$_{\pm 0.017}$ & 0.137$_{\pm 0.006}$ & 0.501$_{\pm 0.015}$ & 0.653$_{\pm 0.016}$ & 0.067$_{\pm 0.004}$ & 0.542$_{\pm 0.019}$ & 0.487$_{\pm 0.014}$ & 0.006$_{\pm 0.001}$ & 0.492$_{\pm 0.016}$ \\
\quad + CGenGA & 0.575$_{\pm 0.019}$ & 0.071$_{\pm 0.004}$ & 0.462$_{\pm 0.016}$ & \underline{0.620$_{\pm 0.015}$} & 0.053$_{\pm 0.004}$ & \underline{0.581$_{\pm 0.016}$} & \underline{0.916$_{\pm 0.019}$} & \underline{0.712$_{\pm 0.018}$} & \underline{0.682$_{\pm 0.015}$} & \underline{0.825$_{\pm 0.019}$} & \underline{0.421$_{\pm 0.011}$} & \underline{0.632$_{\pm 0.018}$} & \underline{0.501$_{\pm 0.015}$} & \underline{0.012$_{\pm 0.002}$} & \underline{0.501$_{\pm 0.006}$} \\
\quad + BAED     & \textbf{0.702}$_{\pm \textbf{0.017}}$ & \textbf{0.141}$_{\pm \textbf{0.006}}$ & \textbf{0.570}$_{\pm \textbf{0.006}}$ & \textbf{0.715}$_{\pm \textbf{0.017}}$ & \textbf{0.075}$_{\pm \textbf{0.005}}$ & \textbf{0.652}$_{\pm \textbf{0.015}}$ & \textbf{0.982}$_{\pm \textbf{0.020}}$ & \textbf{0.886}$_{\pm \textbf{0.021}}$ & \textbf{0.713}$_{\pm \textbf{0.017}}$ & \textbf{0.901}$_{\pm \textbf{0.010}}$ & \textbf{0.814}$_{\pm \textbf{0.017}}$ & \textbf{0.722}$_{\pm \textbf{0.009}}$ & \textbf{0.666}$_{\pm \textbf{0.017}}$ & \textbf{0.016}$_{\pm \textbf{0.002}}$ & \textbf{0.539}$_{\pm \textbf{0.019}}$ \\
\bottomrule
\end{tabular}
}

\end{table*}

\begin{table*}[ht]
\centering
\caption{Overall comparison across two datasets in the inductive setting under dynamic graphs, with bold numbers indicating the best performance and underlined numbers highlighting the most competitive results(RQ2).}
\label{tab:main_dynamic_result}
\resizebox{\linewidth}{!}{
\begin{tabular}{l|l|cccccccc}
\toprule
\multirow{2}{*}{\textbf{Dataset}} & \multirow{2}{*}{\textbf{Method}}  & \multicolumn{2}{c}{\textbf{T=1}} & \multicolumn{2}{c}{\textbf{T=2}} & \multicolumn{2}{c}{\textbf{T=3}} & \multicolumn{2}{c}{\textbf{T=4}} \\ 
\cmidrule(r){3-4} \cmidrule(r){5-6} \cmidrule(r){7-8} \cmidrule(r){9-10} 
 & & \textbf{AUROC} & \textbf{AUPRC} & \textbf{AUROC} & \textbf{AUPRC} & \textbf{AUROC} & \textbf{AUPRC} & \textbf{AUROC} & \textbf{AUPRC} \\
\midrule
\multirow{5}{*}{Dgraph} 
& GCN & 0.4327$_{\pm 0.0105}$ & 0.0092$_{\pm 0.0113}$ & 0.4882$_{\pm 0.0127}$ & 0.0104$_{\pm 0.0142}$ & 0.5091$_{\pm 0.0139}$ & 0.0115$_{\pm 0.0128}$ & 0.5199$_{\pm 0.0116}$ & 0.0118$_{\pm 0.0097}$ \\
& \quad +AEGIS & \underline{0.4684$_{\pm 0.0132}$} & 0.0028$_{\pm 0.0087}$ & \underline{0.5046$_{\pm 0.0115}$} & 0.0045$_{\pm 0.0091}$ & \underline{0.5352$_{\pm 0.0128}$} & 0.0042$_{\pm 0.0083}$ & \underline{0.5492$_{\pm 0.0143}$} & 0.0055$_{\pm 0.0096}$ \\
& \quad +GGAD & 0.4471$_{\pm 0.0109}$ & 0.0019$_{\pm 0.0084}$ & 0.4986$_{\pm 0.0127}$ & 0.0037$_{\pm 0.0089}$ & 0.5169$_{\pm 0.0136}$ & 0.0043$_{\pm 0.0094}$ & 0.5441$_{\pm 0.0118}$ & 0.0057$_{\pm 0.0086}$ \\
& \quad +CGenGA & 0.4524$_{\pm 0.0145}$ & \underline{0.0095$_{\pm 0.0102}$} & 0.5001$_{\pm 0.0139}$ & \underline{0.0109$_{\pm 0.0116}$} & 0.5136$_{\pm 0.0122}$ & \underline{0.0118$_{\pm 0.0108}$} & 0.5377$_{\pm 0.0147}$ & \underline{0.0124$_{\pm 0.0121}$} \\
& \quad +BAED & \textbf{0.4817$_{\pm 0.0082}$} & \textbf{0.0102$_{\pm 0.0051}$} & \textbf{0.5249$_{\pm 0.0069}$} & \textbf{0.0115$_{\pm 0.0047}$} & \textbf{0.5417$_{\pm 0.0073}$} & \textbf{0.0128$_{\pm 0.0039}$} & \textbf{0.5623$_{\pm 0.0091}$} & \textbf{0.0136$_{\pm 0.0054}$} \\
\midrule
\multirow{5}{*}{Elliptic} 
& GCN & 0.4927$_{\pm 0.0091}$ & 0.0317$_{\pm 0.0135}$ & 0.5123$_{\pm 0.0123}$ & 0.0525$_{\pm 0.0148}$ & 0.5087$_{\pm 0.0118}$ & 0.0478$_{\pm 0.0132}$ & 0.5139$_{\pm 0.0106}$ & 0.0496$_{\pm 0.0129}$ \\
& \quad +AEGIS & 0.4977$_{\pm 0.0093}$ & 0.0542$_{\pm 0.0128}$ & 0.5275$_{\pm 0.0104}$ & 0.0701$_{\pm 0.0135}$ & 0.5128$_{\pm 0.0117}$ & 0.0694$_{\pm 0.0142}$ & 0.5318$_{\pm 0.0098}$ & 0.0867$_{\pm 0.0126}$ \\
& \quad +GGAD & \underline{0.5083$_{\pm 0.0125}$} & \underline{0.0643$_{\pm 0.0149}$} & 0.0542$_{\pm 0.0136}$ & \underline{0.0721$_{\pm 0.0153}$} & \underline{0.5727$_{\pm 0.0141}$} & \underline{0.0804$_{\pm 0.0138}$} & \underline{0.5974$_{\pm 0.0150}$} & \underline{0.1084$_{\pm 0.0147}$} \\
& \quad +CGenGA & 0.5017$_{\pm 0.0108}$ & 0.0438$_{\pm 0.0119}$ & \underline{0.5429$_{\pm 0.0143}$} & 0.0508$_{\pm 0.0132}$ & 0.5875$_{\pm 0.0135}$ & 0.0621$_{\pm 0.0124}$ & 0.5911$_{\pm 0.0116}$ & 0.0742$_{\pm 0.0139}$ \\
& \quad +BAED & \textbf{0.5818$_{\pm 0.0043}$} & \textbf{0.1271$_{\pm 0.0026}$} & \textbf{0.6036$_{\pm 0.0038}$} & \textbf{0.1442$_{\pm 0.0031}$} & \textbf{0.5924$_{\pm 0.0029}$} & \textbf{0.1303$_{\pm 0.0025}$} & \textbf{0.6083$_{\pm 0.0035}$} & \textbf{0.1469$_{\pm 0.0028}$} \\
\bottomrule
\end{tabular}
}
\end{table*}

\section{Experiments}
\label{sec:experiment}
In this section, we answer the following five questions. 
\begin{itemize}[leftmargin=*]
    \item \textbf{RQ1}: What is BAED's overall performance compared to state-of-the-art methods?
    \item \textbf{RQ2}: How does BAED perform on dynamic graph datasets compared to state-of-the-art methods?
    \item \textbf{RQ3}: How does Anomaly-Guidance Embedding improve the performance of graph anomaly detection?
    \item \textbf{RQ4}: How does the weighting factor in Curriculum Anomaly Augmentation affect the performance of denoising models?
    \item \textbf{RQ5}: What advantages does BAED exhibit in generative data augmentation?
    \item \textbf{RQ6}: Is the distribution of the generated samples similar to the distribution of the original samples?
    
\end{itemize}

\subsection{Experimental Setups}
\subsubsection{Datasets}
In the experiments, we evaluate the performance of our approach on five large-scale, real-world graph datasets from diverse domains. Specifically, the datasets used include the transaction network from T-Finance~\cite{tang2022rethinking}, the social interaction network from Reddit~\cite{kumar2019predicting}, the Bitcoin transaction network from Elliptic~\cite{weber2019anti}, the co-purchase network from Photo~\cite{mcauley2015image}, and the financial network from DGraph~\cite{huang2022dgraph}. Detailed statistics are provided in Appendix~\ref{Datasets}.

\subsubsection{Competing Methods}
To validate the effectiveness of the BAED model, we compared it with three state-of-the-art generative anomaly detection models: AEGIS~\cite{ding2021inductive}, GGAD~\cite{qiao2024generative}, and CGenGA~\cite{ma2024graph}. 
These models represent the latest advancements in GAD from the perspective of anomalous sample generation. Detailed information about these models can be found in Appendix~\ref{Baselines}.

\subsubsection{Evaluation Metric}
Based on~\cite{qiao2024generative, pang2021toward}, three commonly used complementary evaluation metrics are employed in this study: the Area Under the ROC Curve (AUROC), the Area Under the Precision-Recall Curve (AUPRC), and the F1-score. Higher AUROC, AUPRC, and F1-score values indicate better model performance. 

\subsubsection{Implementation Details}
Experiments are conducted on four baseline GNN models (GCN~\cite{kipf2016semi}, BWGNN~\cite{tang2022rethinking}, GraphSAGE~\cite{hamilton2017inductive}, and BernNet~\cite{he2021bernnet}) and four negative sample generators (GGAD~\cite{qiao2024generative}, AEGIS~\cite{ding2021inductive}, CGenGA~\cite{ma2024graph}, and BAED). 
The datasets are split into training (70\%), validation (15\%), and test (15\%) sets.
Unlike traditional transductive graph learning, our approach adopts an inductive learning framework, where the model does not have access to the test set during training. The graph is transformed into ego-graph for each node, with the ego-graph of test nodes excluded from the training set. This design promotes better generalization to unseen data, which is a key advantage of inductive learning.
The Adam optimizer is used, and the learning rate is tuned within the range of $[10^{-3}, 10^{-4}, 10^{-5}, 10^{-6}]$. 

\subsection{Improvement on Anomaly Detection~(RQ1)}

Table~\ref{tab:main_result} demonstrates BAED's superior performance over existing data augmentation methods. BAED-generated synthetic data enhances traditional GNN models by up to twice their original performance, showcasing strong generalization capabilities.

BAED exhibits exceptional robustness on challenging datasets. On the sparse Elliptic dataset, while AEGIS and GGAD degraded BWGNN's AUROC by 12.99\% and 15.01\% respectively, BAED achieved a 4.97\% improvement. More dramatically, on the large T-Finance dataset, BAED boosted BWGNN's AUROC and AUPRC by 90.93\% and 84.84\%, substantially outperforming competing methods. These results indicate that BAED generates data better aligned with true anomaly distributions, enhancing detector robustness in noisy and sparse environments.

\subsection{Performance on Dynamic Graph Datasets (RQ2)}

Table~\ref{tab:main_dynamic_result} presents the experimental results of BAED and baseline methods under an incremental learning paradigm on dynamic graph datasets, which simulates real-world streaming data scenarios by sequentially training and testing models on consecutive temporal segments to mimic non-stationary graph environments. BAED demonstrates superior and stable performance across all time steps on both datasets, exhibiting significantly lower variance in its results compared to other methods, which highlights its robust adaptability to evolving graph structures.

\subsection{The Impact of Anomaly-Guidance~(RQ3)}
To evaluate the Anomaly-Guidance Embedding in \textbf{BAED}, we compared it with two variants: \textbf{BUED}, a balanced unconditional ego-graph diffusion model, and the \textbf{original} model, which lacks data augmentation. Experiments on the Reddit and Photo datasets (Figure~\ref{fig:ep2}) demonstrate that both BAED and BUED, with curriculum anomaly augmentation, outperform the original model trained on imbalanced data. This highlights the crucial role of specific anomaly self-graphs in enhancing detection performance. Furthermore, BAED outperforms BUED, validating the effectiveness of our dynamic anomaly-guided graph generation technique, which customizes augmented samples based on specific anomaly patterns and improves model performance.

\begin{figure}[t]
    \centering
    \begin{subfigure}{0.49\linewidth}
        \centering
        \includegraphics[width=\linewidth]{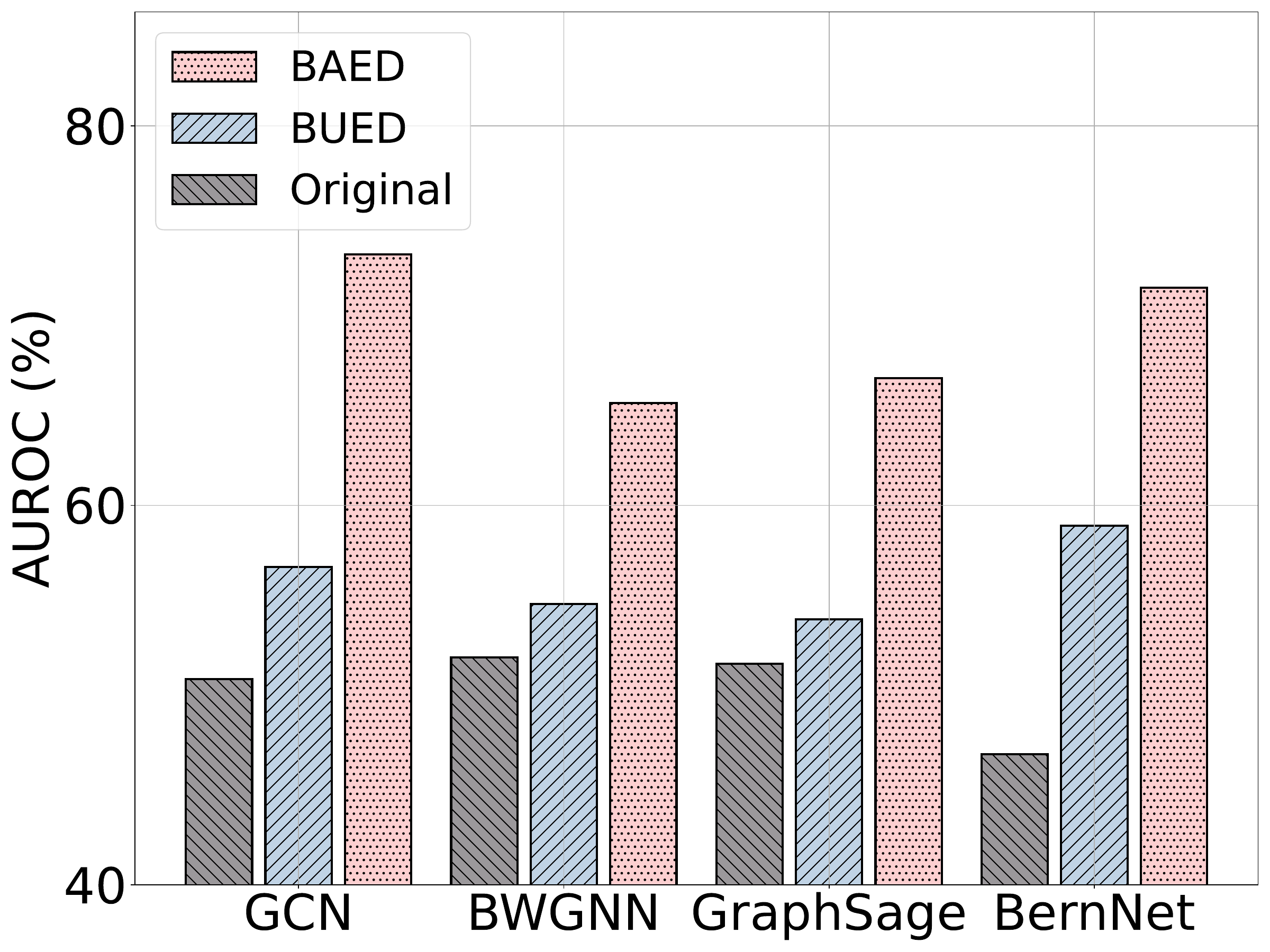} 
        \caption{\footnotesize Comparison of AUROC for Reddit}
        \label{fig3-a}
    \end{subfigure}
    \hfill 
    \begin{subfigure}{0.49\linewidth}
        \centering
        \includegraphics[width=\linewidth]{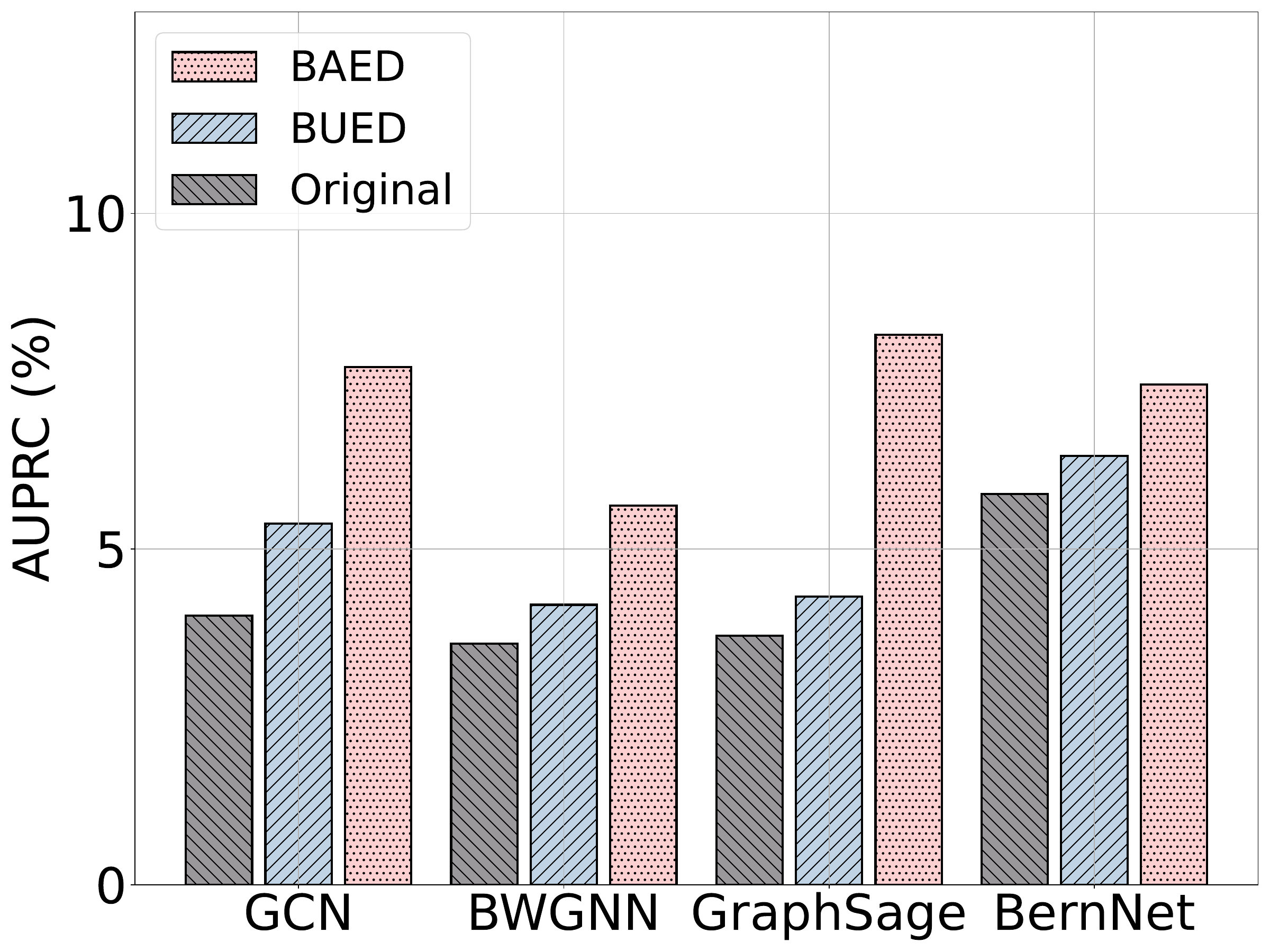} 
        \caption{\footnotesize Comparison of AUPRC for Reddit}
        \label{fig3-b}
    \end{subfigure}
    \vfill 
    \begin{subfigure}{0.49\linewidth}
        \centering
        \includegraphics[width=\linewidth]{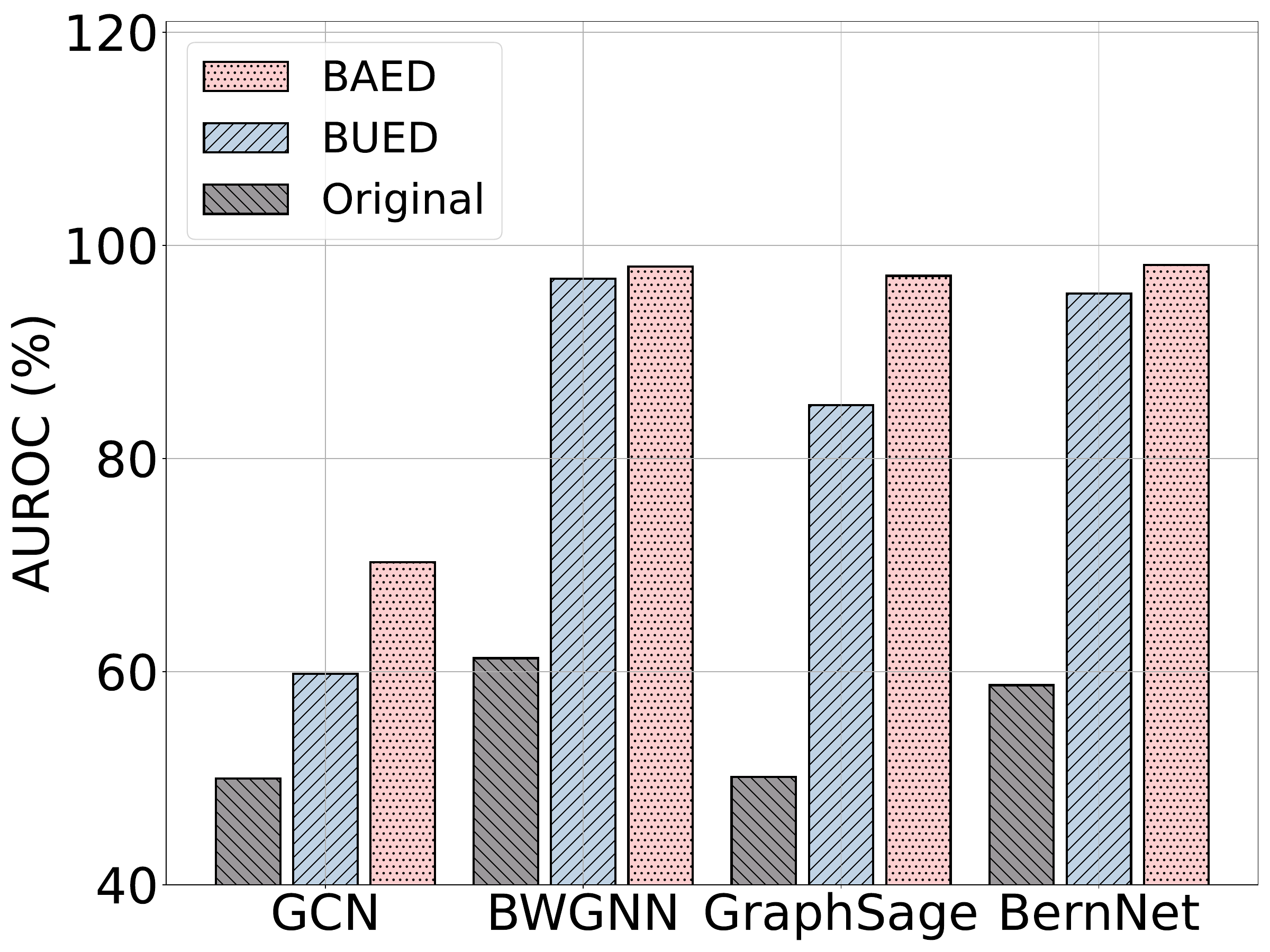} 
        \caption{\footnotesize Comparison of AUROC for Photo}
        \label{fig3-c}
    \end{subfigure}
    \hfill 
    \begin{subfigure}{0.49\linewidth}
        \centering
        \includegraphics[width=\linewidth]{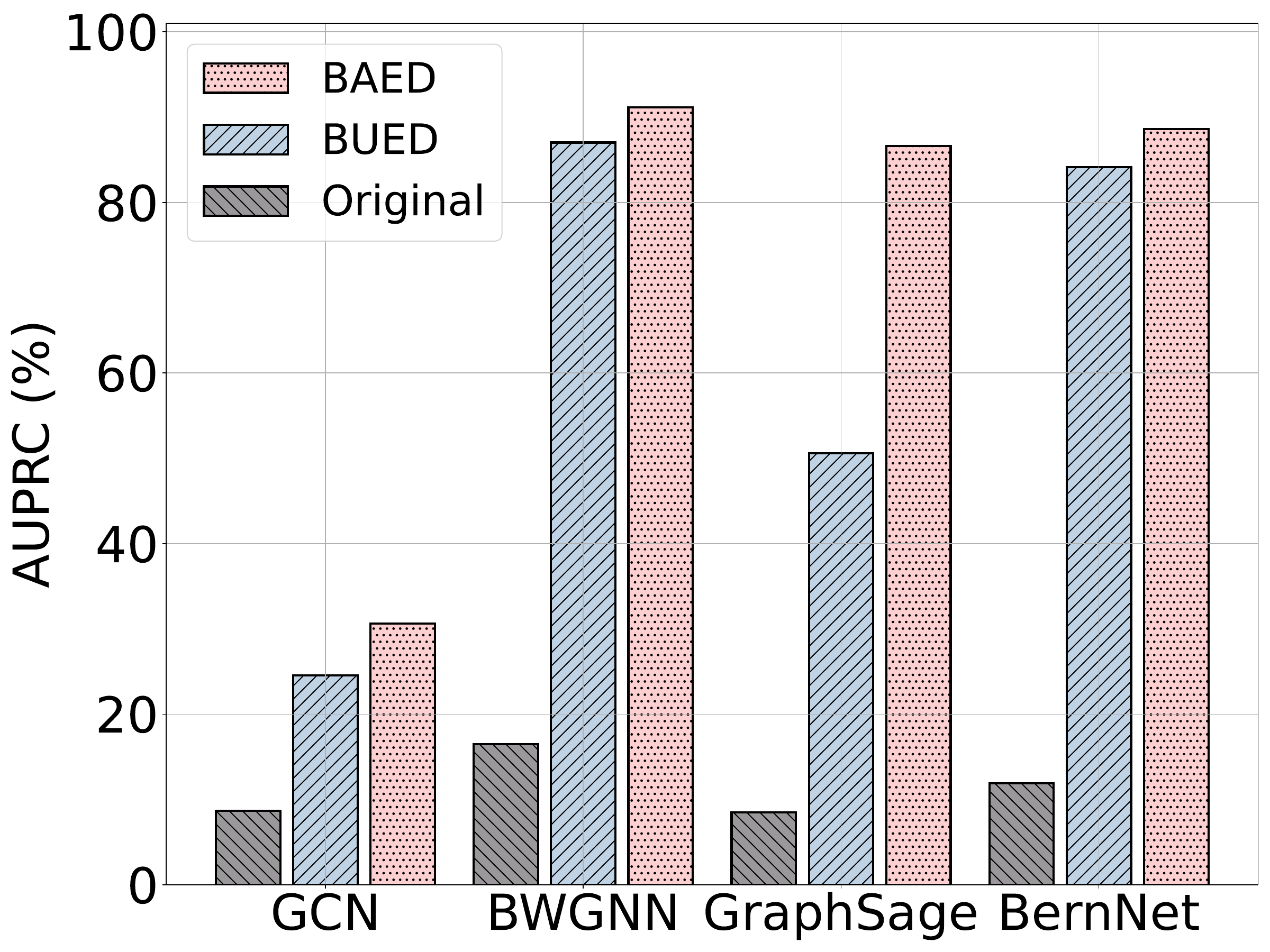} 
        \caption{\footnotesize Comparison of AUPRC for Photo}
        \label{fig3-d}
    \end{subfigure}
    \caption{Impact of Anomaly-Guidance Embedding.}
    \label{fig:ep2}
\end{figure}

\subsection{Ablation Study on Curriculum Anomaly Augmentation~(RQ4)}
\begin{table}[ht]
\centering
\setlength{\tabcolsep}{5pt} 
\renewcommand{\arraystretch}{1.2} 
\caption{Performance on Different Guidance Embedding}

\label{tab:ablation}
\resizebox{0.95\linewidth}{!}{
\begin{tabular}{@{}c|c|cccc@{}}
\toprule
\multirow{2}{*}{\textbf{Model}} & \multirow{2}{*}{\textbf{Weighting Strategy}} & \multicolumn{2}{c}{\textbf{Reddit}} & \multicolumn{2}{c}{\textbf{Elliptic}} \\ 
\cmidrule(lr){3-4} \cmidrule(lr){5-6}
 &  & \textbf{AUROC} & \textbf{AUPRC} & \textbf{AUROC} & \textbf{AUPRC} \\ 
\midrule
\multirow{3}{*}{GCN}      & random   & 0.6075 & 0.0405 & 0.5745 & 0.1329 \\
          & average  & 0.6225 & 0.0696 & 0.6016 & 0.1437 \\
          & \textbf{curriculum-based} & \textbf{0.7322} & \textbf{0.0771} & \textbf{0.6706} & \textbf{0.1537} \\
          \midrule
\multirow{3}{*}{BWGNN}       & random   & 0.5316 & 0.0425 & 0.5713 & 0.0852 \\
          & average  & 0.5642 & 0.0401 & 0.5878 & 0.0761 \\
          & \textbf{curriculum-based} & \textbf{0.6538} & \textbf{0.0565} & \textbf{0.6689} & \textbf{0.1019} \\
          \midrule
\multirow{3}{*}{GraphSage}    & random   & 0.5954 & 0.0406 & 0.5729 & 0.0957 \\
          & average  & 0.5979 & 0.0405 & 0.5898 & 0.1312 \\
          & \textbf{curriculum-based} & \textbf{0.6669} & \textbf{0.0819} & \textbf{0.5927} & \textbf{0.1340} \\
          \midrule
\multirow{3}{*}{BernNet}    & random   & 0.5565 & 0.0490 & 0.5150 & 0.0828 \\
          & average  & 0.5930 & 0.0384 & 0.5875 & 0.1150 \\
          & \textbf{curriculum-based} & \textbf{0.7147} & \textbf{0.0745} & \textbf{0.7016} & \textbf{0.1412} \\
\bottomrule
\end{tabular}
}
\end{table}

To evaluate the Curriculum Anomaly Augmentation strategy, we conducted an ablation study on four baseline models, focusing on the impact of different anomaly-guided embedding methods. We tested three weight assignment strategies: random, average, and curriculum-based, which dynamically adjusts embeddings for specific anomaly patterns. The results in Table \ref{tab:ablation} show that the curriculum-based strategy outperforms both random and average methods across all metrics. The latter failed to prioritize anomaly patterns at each training stage, resulting in less effective guidance. In contrast, the dynamic weight adjustment significantly improved performance by emphasizing core anomalies.

Additionally, training loss curves for the Reddit dataset (Figure \ref{fig:ep3-loss}) show that the curriculum-based method achieves smoother and lower loss convergence, demonstrating more effective guidance during training.

\begin{figure}[!t]
  \centering
  \includegraphics[width=0.99\linewidth]{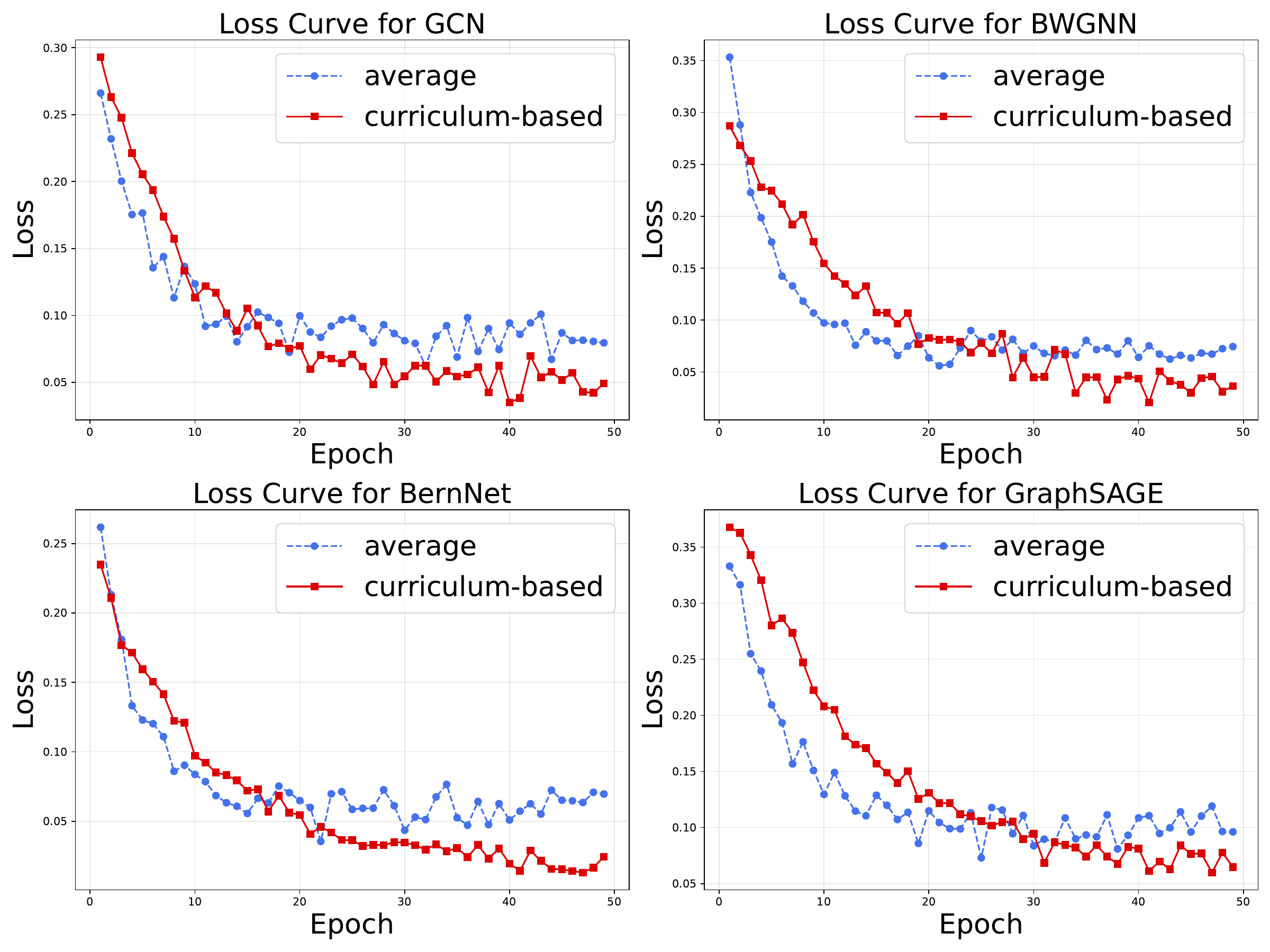}
  \caption{Comparison of model loss curves for different weighting strategy on reddit dataset}
  \label{fig:ep3-loss}
\end{figure}

\subsection{The Impact of BAED on data diversity~(RQ5)}
To evaluate the diversity of anomalous samples generated by BAED, we compared BAED-generated ego-graphs with randomly sampled ones for balancing positive and negative samples. Results in Figure~\ref{fig:ep4} show that BAED-generated ego-graphs significantly outperform random samples in anomaly detection, offering greater sample diversity. This highlights the effectiveness of BAED’s diffusion-based mechanism in generating anomaly-specific ego-graphs that capture the local structure of anomalous nodes, enhancing sample diversity and improving model generalization.

\begin{figure}[!t]
  \centering
  \includegraphics[width=0.99\linewidth]{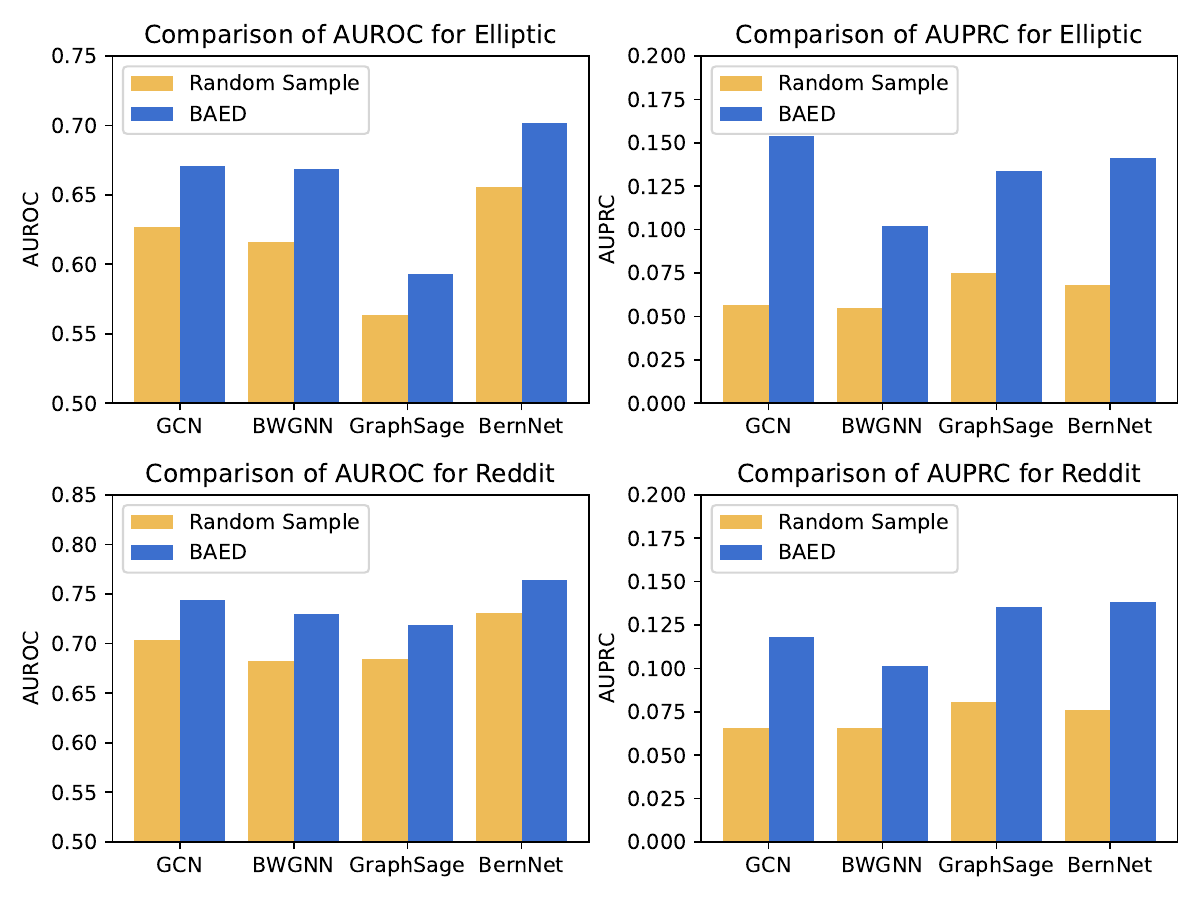}
  \caption{Comparison of different augmentation methods}
  \label{fig:ep4}
\end{figure}

\subsection{Visualization of Data Distribution~(RQ6)}
\begin{figure}[t]
	\centering
	\begin{subfigure}{0.49\linewidth}
		\centering
		\includegraphics[width=\linewidth]{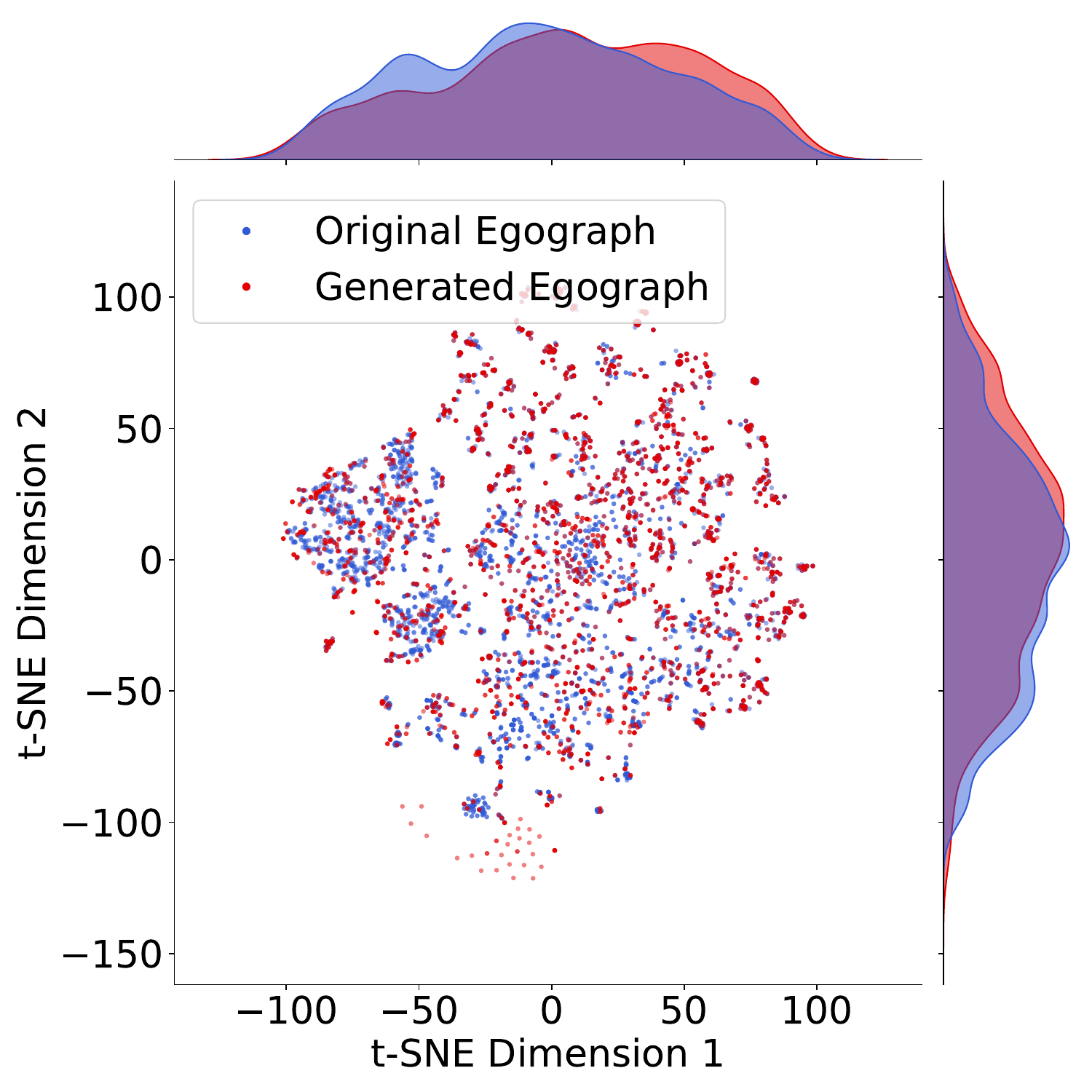}
		\caption{elliptic}
		\label{fig:ep5-a}
	\end{subfigure}
	\centering
	\begin{subfigure}{0.49\linewidth}
		\centering
		\includegraphics[width=\linewidth]{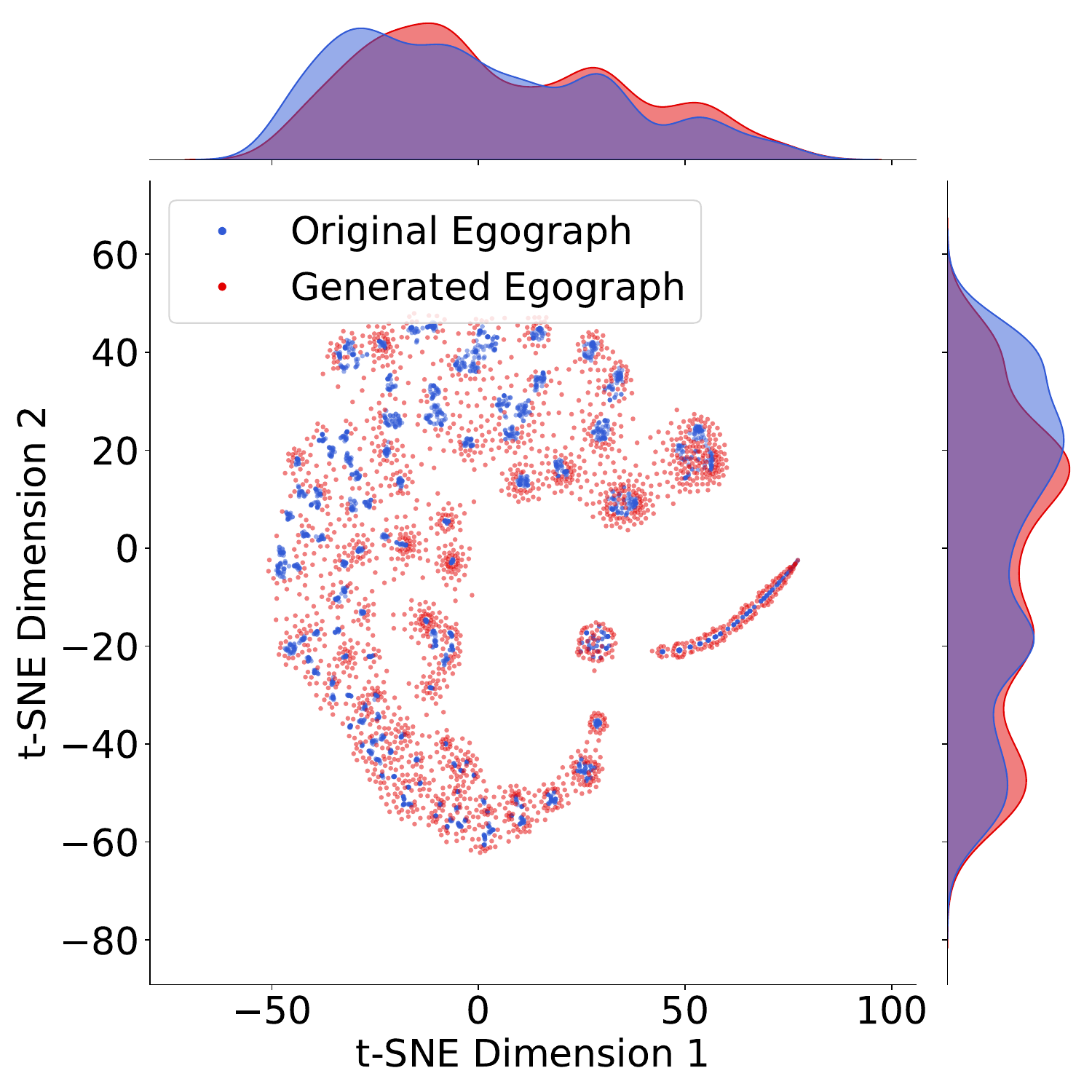}
		\caption{reddit}
		\label{fig:ep5-b}
	\end{subfigure}
	\centering
 \caption{Comparison of data distribution.}
 
 \label{fig:ep5}
\end{figure}

To compare the distribution of embeddings between generated and original anomalous ego-graphs, we used t-SNE for dimensionality reduction (Figure \ref{fig:ep5}). The results show that the embeddings of BAED-generated anomalous samples closely match those of the original samples, with nearly identical distributions. This indicates that BAED effectively preserves the distributional characteristics of the original data through its discrete diffusion process, ensuring the validity of the generated samples.

%% file: data/conclusion.tex
\section{Conclusions and Limitations}
\label{sec:conclusion}
This paper presented \textbf{Balanced Anomaly-Guided Ego-graph Diffusion} (BAED), a novel framework that addresses the dual challenges of dynamic graph structures and class imbalance in inductive graph anomaly detection. By combining a discrete ego-graph diffusion model with curriculum anomaly augmentation, BAED achieves dynamic adaptability and enhanced generalization to unseen graph data. Extensive experiments demonstrate its state-of-the-art performance, underscoring the potential of integrating topology-aware augmentation with adaptive learning strategies for anomaly detection tasks.

While BAED introduces significant advancements, it has a few limitations that warrant future exploration. First, the generated ego-graphs primarily focus on structural anomalies. Expanding the framework to address more nuanced feature-based anomalies could broaden its applicability. Second, BAED has been validated in the context of extreme class imbalance for anomaly detection tasks. Exploring its potential to handle other imbalance scenarios, such as imbalanced ego-graph classification or rare event prediction, could further demonstrate its versatility.

%% file: data/appendix.tex
\section{Algorithm}
\label{app:algorithm}

The implementation of BAED consists of three core phases: \textbf{Pretrained GAE} for representation learning; \textbf{Diffusion Model Training and Sample Generation} as detailed in Algorithms \ref{algo:inference} and \ref{alg:training_edge_pred_subgraph}; and \textbf{Anomaly Detection} according to the procedure in Algorithm~\ref{alg:gad}.



\begin{algorithm}[htbp]
\caption{BAED: Training of Diffusion}
\label{alg:training_edge_pred_subgraph}
\begin{algorithmic}[1]
\STATE \textbf{Input:} node features \( X^{0} \), topological structure \( A^{0} \) , diffusion steps \( T \), forward transition distributions \( q(A^{t} | A^{t-1}) \) and \( q(A^{t} | A^{0}) \), ego-graph embedding function \( \mathbf{e}^{t} = f_{\theta}(A^{t}) \).
\WHILE{training}
    \STATE Compute ego-graph embedding \( \mathbf{e}^{0} = f_{\theta}(X^{0},A^{0}) \)
    \STATE Sample \( t \sim \mathcal{U}[1, T] \), and set \( t - 1 \) immediately
    \STATE Draw \( A^{t-1} \sim q(A^{t-1} | A^{0}) \), then draw \( A^{t} \sim q(A^{t} | A^{t-1}) \)
    \STATE Perform gradient descent over \( \theta \) with gradient \( \nabla_{\theta} \log p_{\theta}(A^{t-1} | A^{t}, \mathbf{e}^{0}) \)
\ENDWHILE
\end{algorithmic}
\end{algorithm}

\begin{algorithm}[htbp]
\caption{Training of Graph Anomaly Detection}
\label{alg:gad}
\begin{algorithmic}[1]
\STATE  Initialize training parameters: learning rate \( \text{lr} \), number of epochs \( E \), optimizer \( \text{Adam} \), BAED denoising model.  
\STATE Split dataset into training set, validation set, and test set.
\FOR{$e = 1$ to $E$}
    \FOR{each batch \( \mathcal{B} \) in training set}
        \STATE Compute the difference in the number of positive and negative samples \(N\)
        \STATE Compute the anomaly-guidance embedding for negative samples in the batch.
        \IF{$e > 0$}
            \STATE Compute weight factor \( w_i \) and time factor \( h(t) \) for each sample \( i \) in the batch:\( w_i = \frac{1}{1 + e^{-\alpha (l_i - \beta)}} \cdot h(t) \).
        \ELSE
            \STATE Set weight factor \( w_i \) to the average value for all samples.
        \ENDIF
        \STATE Compute weighted embedding \( \tilde{\mathbf{h}}^{g} = \sum_{i=1}^N \frac{w_i}{\sum_{j=1}^N w_j} \cdot \text{GIN}(\mathcal{G}_K^i) \).
        \STATE Generate \( N \) negative samples using BAED denoising model with \(\tilde{\mathbf{h}}^{g}\).
        \STATE Combine the original negative samples with the generated samples in the batch.
        \STATE Compute cross-entropy loss \( \nabla_{\theta} \mathcal{L} \) to update model.
    \ENDFOR
\ENDFOR
\STATE \RETURN trained model.
\end{algorithmic}
\end{algorithm}


\section{Dataset Overview}
\label{Datasets}
We utilize five public datasets from diverse domains: Dgraph, T-Finance, Photo, Reddit, and Elliptic. Each exhibits a significant class imbalance between positive and negative samples, aligning with our research focus. Details are summarized in Table \ref{tab:dataset}.

\begin{table}[t]
\centering
\caption{The statistics of the graph anomaly detection datasets}
\label{tab:dataset}
\resizebox{\linewidth}{!}{ 
\begin{tabular}{ccccccc}
\toprule
\textbf{Dataset} & \#\textbf{Nodes} & \#\textbf{Edges} & \#\textbf{Features} & \#\textbf{Degree} & \#\textbf{Anomalies} & \textbf{Anomaly Ratio} \\
\midrule
Dgraph   & 3,700,550   & 4,300,999   & 17    & 1.16   & 15,509  & 1.3\%   \\
T-Finance & 39,357      & 21,222,543  & 10    & 539.23 & 1,803   & 4.6\%   \\
Photo    & 7,535       & 119,043     & 745   & 15.8   & 698     & 9.2\%   \\
Reddit   & 10,984      & 168,016     & 64    & 15.3   & 366     & 3.3\%   \\
Elliptic & 203,769     & 234,355     & 166   & 1.15   & 4,545   & 9.8\%   \\
\bottomrule
\end{tabular}
}
\end{table}

\section{Baseline Methods}
\label{Baselines}
In this section, we present and compare the properties of the baseline methods used in this paper.
\begin{itemize}
    \item \textbf{AEGIS}~\cite{ding2021inductive}: AEGIS introduces a graph discretization layer to learn anomaly-aware node representations through an attentional mechanism. It employs generative adversarial learning to enable inductive anomaly detection without retraining.
    
    \item \textbf{GGAD}~\cite{qiao2024generative}: GGAD is a semi-supervised method that enhances detection by generating pseudo-outlier nodes as negative samples. It optimizes these nodes using asymmetric local affinity and egocentric proximity priors to align with real anomalies.

    \item \textbf{CGenGA}~\cite{ma2024graph}: CGenGA utilizes a denoising diffusion model to generate synthetic nodes matching the original graph distribution. It enhances existing detectors by optimizing the generation process through classifier guidance.
\end{itemize}

\section{Experimental Details}
\label{ap:details_of_exp}
All experiments run on an NVIDIA Tesla A800 GPU using PyTorch and Adam optimizer. Results are averaged over five runs with random seeds {15, 42, 63, 87, 94}. 

\textbf{Diffusion Model.} Our conditional-guided generator uses node degree and an 8-head self-attention mechanism. It trains for 300 epochs with a diffusion dimension of 64, 128 steps, a learning rate of 1e-4, linear noise scheduling, and a dropout rate of 0.1.

\textbf{Training.} Batches are balanced using generated negative samples, while the test set retains the original imbalance. Synthetic node features combine original values with Gaussian noise $\mathcal{N}(0, 1)$. Downstream models train for 50 epochs, with learning rates tuned in {$10^{-3}$–$10^{-6}$}.

\section{Computational Efficiency for Real-time Graph Anomaly Detection}
\label{app:efficiency}
We analyze the time complexity of BAED's training phase. The per-epoch complexity is:
\begin{equation}
\mathcal{O}(N \cdot L \cdot \bar{e} \cdot d + B \cdot T \cdot \bar{e} \cdot d)
\end{equation}
where $N$ is the number of ego-graphs, $L$ the GNN layers, $\bar{e}$ the average edges per ego-graph, $d$ the feature dimension, $B$ the batch size, and $T$ the diffusion steps.

Therefore, BAED based on diffusion model demands high computational efficiency, especially for large-scale dynamic networks. We evaluate the runtime performance of all methods across five datasets under GPU acceleration, with results presented in Table~\ref{tab:runtimes}.

\begin{table}[ht]
\centering
\caption{Runtime comparison (seconds) across datasets on GPU}
\label{tab:runtimes}
\begin{tabular}{@{}lccccc@{}}
\toprule
\textbf{Model} & \textbf{Elliptic} & \textbf{Dgraph} & \textbf{Reddit} & \textbf{Photo} & \textbf{T-Finance} \\
\midrule
AEGIS & 5,761 & 2,247 & 261 & 498 & 15,258 \\
GGAD & 5,271 & 2,713 & 463 & 487 & 9,345 \\
CGenGA & 11,233 & 14,474 & 1,598 & 2,986 & 30,650 \\
BAED & \textbf{8,641} & \textbf{10,802} & \textbf{1,142} & \textbf{2,489} & \textbf{27,864} \\
\bottomrule
\end{tabular}
\end{table}
While BAED requires more time than lighter baselines (AEGIS, GGAD), it achieves a 27.98\% average runtime reduction compared to CGenGA. This gain is due to our anomaly-guided generation, which focuses the diffusion process. The improved detection performance justifies the computational cost for accuracy-critical applications like fraud detection.